\documentclass[10pt]{article}
\usepackage[preprint]{tmlr}

\usepackage{amsmath,amsthm,amssymb,bm,booktabs,tabularx}
\usepackage{graphicx,algorithm,algorithmic,enumitem,pifont}
\usepackage{hyperref,url}
\usepackage[nameinlink,noabbrev]{cleveref}

\newtheorem{theorem}{Theorem}[section]
\newtheorem{lemma}[theorem]{Lemma}
\newtheorem{definition}[theorem]{Definition}
\newtheorem{corollary}[theorem]{Corollary}

\theoremstyle{remark}

\theoremstyle{plain}
\makeatletter
\@addtoreset{algorithm}{section}
\makeatother

\usepackage{amsmath}
\usepackage{amsfonts}
\usepackage{amssymb}
\usepackage{mathabx}
\usepackage{mathtools}
\usepackage{etoolbox}

\newcounter{constant} 
\newcommand{\Alg}{\texttt{Alg}_{\texttt{off}}}

\newtheorem{assumption}{Assumption}
\let\hat\widehat
\let\tilde\widetilde

\usepackage{tikz}

\newcommand{\abs}[1]{\lvert#1\rvert}

\newcommand{\cA}{\mathcal{A}}
\newcommand{\cB}{\mathcal{B}}

\newcommand{\cE}{\mathcal{E}}
\newcommand{\cF}{\mathcal{F}}
\newcommand{\cG}{\mathcal{G}}
\newcommand{\cH}{\mathcal{H}}

\newcommand{\cM}{\mathcal{M}}

\newcommand{\cP}{\mathcal{P}}

\newcommand{\cS}{{\mathcal{S}}}
\newcommand{\cT}{{\mathcal{T}}}

\newcommand{\cX}{\mathcal{X}}
\newcommand{\cY}{\mathcal{Y}}

\newcommand{\EE}{\mathbb{E}}

\newcommand{\II}{\mathbb{I}}

\newcommand{\PP}{\mathbb{P}}

\newcommand{\QQ}{\mathbb{Q}}
\newcommand{\RR}{\mathbb{R}}

\newcommand{\hatreg}{\hat{\mathrm{reg}}}

\newcommand{\argmax}{\mathop{\mathrm{argmax}}}

\newcommand{\tr}{\mathop{\mathrm{Tr}}}

\usepackage{esint}

\newcommand{\inner}[2]{\left\langle #1,#2 \right\rangle}
\newcommand{\rbr}[1]{\left(#1\right)}

\newcommand{\cbr}[1]{\left\{#1\right\}}

\renewcommand{\hat}[1]{\widehat{#1}}

\newcommand{\pll}{\kern 0.3em/\kern -0.9em /\kern 0.3em}

\AtBeginEnvironment{algorithmic}{\let\AND\relax}

\title{Model-Based Reinforcement Learning with Double Oracle Efficiency in Policy Optimization and Offline Estimation}

\author{\name Haichen Hu \email huhc@mit.edu \\
\addr Massachusetts Institute of Technology
\AND
\name Jian Qian \email jianqian@hku.hk \\
\addr The University of Hong Kong
\AND
\name David Simchi-Levi \email dslevi@mit.edu \\
\addr Massachusetts Institute of Technology}

\hypersetup{pdftitle={Model-Based Reinforcement Learning with Double Oracle Efficiency in Policy Optimization and Offline Estimation},pdfauthor={},pdfsubject={},pdfkeywords={reinforcement learning,offline oracles,log determinant}}

\begin{document}
\maketitle

\begin{abstract}
Reinforcement learning (RL) in large environments often faces severe computational bottlenecks, since standard regret-minimization algorithms require repeated and expensive calls to planning and statistical estimation oracles. Although recent work has developed offline-oracle-efficient algorithms, their computational complexity typically scales with the cardinality of the state and action spaces, making them impractical in large-scale or continuous environments. In this paper, we address this limitation by studying offline-oracle-efficient episodic RL through log-barrier and log-determinant regularization. For tabular Markov Decision Processes (MDPs), we propose a new algorithm that achieves the optimal $\tilde{O}(\sqrt{T})$ regret bound while requiring only $O(H\log\log T)$ calls to both the offline statistical estimation oracle and the planning oracle when $T$ is known, and $O(H\log T)$ calls when $T$ is unknown. Importantly, this oracle complexity is independent of the sizes of the state and action spaces. This independence substantially reduces planning-oracle complexity and improves on existing offline-oracle-efficient algorithms \citep{qian2024offline}. We further demonstrate the flexibility of our framework by extending it to linear MDPs with infinite state spaces and arbitrary action spaces, attaining a sublinear regret at $\tilde{O}(T^{3/4})$. Thus, our work provides the first doubly oracle-efficient regret minimization algorithm with efficiency in both statistical estimation and policy optimization for MDPs with infinite state and action spaces, broadening the scope of reinforcement learning.
\end{abstract}



\section{Introduction}
Grounded in Markov Decision Processes (MDPs), reinforcement learning (RL) has emerged as a central paradigm in modern artificial intelligence. By optimizing sequential decision-making through environmental interaction, RL has achieved remarkable success across diverse complex applications. These range from robotic continuous control \citep{kober2013reinforcement} and autonomous transportation \citep{kiran2021deep}, to classical operations research domains like inventory and supply chain management \citep{balseiro2019dynamic,kastius2022dynamic}. Furthermore, RL is pivotal in aligning Large Language Models (LLMs) with human intent via Reinforcement Learning from Human Feedback (RLHF) \citep{ouyang2022training,dai2023safe}.

In online RL, agents in unknown environments must efficiently balance the exploration-exploitation trade-off to minimize cumulative regret against an optimal benchmark. This sequential challenge of environment estimation and policy optimization has driven extensive research in regret minimization across various MDP models. Prominent solutions include model-based \citep{auer2008near, azar2017minimaxregretboundsreinforcement} and model-free \citep{zhang2020almost,jin2018q} methods for tabular MDPs, as well as function approximation techniques \citep{jin2020provably,dai2023refined,kim2022improved,wang2020reinforcement} for more complex settings such as linear MDPs and kernel-based MDPs \citep{hu2026minimax}.

To address the intertwined challenges of statistical efficiency and computational complexity under function approximation, a common approach in theoretical RL is to decouple the two through an oracle-based framework \citep{foster2021statistical,chen2025unified}. In this framework, algorithms sequentially estimate the unknown environment or value functions from existing data, often by invoking a statistical estimation procedure. The agent then passes the resulting optimistic estimated model to a policy-optimization oracle or planning oracle to compute the next policy \citep{xie2023the,agarwal2019reinforcement}. 

Within this framework, most existing theoretical regret-minimization algorithms require $O(T)$ calls to sequential estimation and policy-optimization oracles. Such frequent oracle calls impose substantial computational and time costs, which can make these algorithms impractical in many applications. Moreover, RL typically relies on online estimation oracles to handle adaptively collected data. In contrast to offline estimation oracles, where the data are collected i.i.d. and are well understood \citep{rakhlin2022mathstat,vershynin2018high}, online estimation oracles are currently available only for relatively limited function classes \citep{hazan2016computational}. Their construction remains unclear for many complex model classes, including deep neural networks.

These limitations motivate the development of RL algorithms that are efficient in both policy optimization and offline estimation. While low-switching-cost algorithms reduce planning overhead \citep{bai2019provably,qiao2022sample,wang2021provably,qiao2023logarithmic,gao2021provably}, they exhibit significant drawbacks. Primarily, they rely on online estimation oracles and incur policy-optimization complexities that scale with the state and action spaces. Furthermore, specific methods suffer from severe bottlenecks, such as enumerating the entire policy version space \citep{qiao2022sample} or yielding regret exponential in the feature dimension \citep{wang2021provably}. Finally, approaches that merely reduce policy execution switches still require per-round model estimation and optimization \citep{gao2021provably}, which fails to achieve true oracle efficiency.

Regarding offline oracle efficiency, contextual-bandit algorithms can achieve $O(\log\log T)$ regression-oracle calls for a known horizon $T$ and finite action spaces \citep{simchi2020bypassing}, with subsequent extensions to functional regression \citep{hu2025contextual} and general action spaces \citep{qin2026taming}. Conversely, while UCB-style methods exist for contextual bandits and tabular MDPs \citep{xu2020upper,levy2024eluder,hu2025constrainedonlinedecisionmakingunified,levy2026nearoptimalregretpolicyoptimization}, they require $O(T)$ calls and thus lack full oracle efficiency. Most notably, the first strongly oracle-efficient algorithm for tabular MDPs was recently introduced \citep{qian2024offline}, requiring only $O(H\log\log T)$ density-estimation and $O(HS^2A\log\log T)$ policy-optimization oracle calls during online interaction.

Nevertheless, existing RL algorithms still have two primary limitations. First, the number of policy-optimization or planning-oracle calls is either $O(T)$ or scales with the cardinality of the state and action spaces. This dependence is computationally prohibitive in environments with large state and action spaces. As a result, existing methods do not achieve policy-optimization oracle efficiency and offline-estimation oracle efficiency simultaneously. Second, the current literature on offline-estimation-oracle-efficient regret minimization in RL is confined to tabular MDPs \citep{qian2024offline}. It therefore remains open whether offline-oracle-efficient algorithms can achieve sublinear regret in more complex MDP models. These limitations motivate the following questions:

1: \emph{Can we design a regret-minimization algorithm that is efficient with respect to both offline estimation and planning oracles, with oracle-call complexity that is entirely independent of the state and action space dimensions?}

2: \emph{Can we further develop offline-oracle-efficient regret-minimization algorithms for environments beyond tabular MDPs with finite state and action spaces?}

\paragraph{Our Contributions.} In this paper, we address both challenges. First, for tabular MDPs with finite state and action spaces, we develop an algorithm that achieves the optimal $\tilde{O}(\sqrt{T})$ regret bound while requiring only $O(H\log\log T)$ oracle calls when $T$ is known and $O(H\log T)$ oracle calls when $T$ is unknown, to both the offline estimation oracle and the policy-optimization oracle. In contrast to existing works \citep{qian2024offline,qiao2023logarithmic}, our oracle-call complexity is independent of the cardinalities of the state and action spaces. Therefore, our algorithm is doubly oracle-efficient, simultaneously with respect to offline estimation and policy optimization. Second, we introduce, to our knowledge, the first offline-oracle-efficient algorithm for linear MDPs with infinite state spaces and arbitrary action spaces  that enjoys the same $O(H\log\log T)$ and $O(H\log T)$ oracle-call complexity when $T$ is known and unknown respectively, with a sublinear regret $\tilde{O}(T^{3/4})$. Moreover, from the computational side, we just need our policy optimization oracle to maximize an estimated value function regularized by a log-barrier or log-determinant, which is much more common and efficient in reinforcement learning \citep{zhang2025constrained,cesani2026log,misra2023bellman,chen2026information},compared with directly maximizing the occupancy measure \citep{qian2024offline}.

\paragraph{Paper Structure.}
The remainder of this paper is organized into two main parts, sequentially addressing tabular and linear MDPs. Section \ref{sec:model_setup_tabular} formalizes the tabular MDP framework. Subsequently, Section \ref{sec:tabular_alg_design} introduces our regret-minimization algorithm for this regime, while Section \ref{sec:theory_tabular} provides the corresponding theoretical guarantees and a proof sketch. We then transition to linear MDPs featuring infinite state and action spaces. Following the model formulation in Section \ref{sec:model_setup_linear}, Section \ref{sec:linear_alg_design} presents our doubly oracle-efficient algorithm, and Section \ref{sec:theory_linear} establishes its rigorous theoretical bounds. Detailed proofs about all the results are provided in the Appendix.
\section{Related Works}
\paragraph{Oracle-based reinforcement learning}  Algorithms using online oracles, such as the E2D framework \citep{foster2020beyond, foster2021statistical} are well-studied, they require $O(T)$ calls to online density estimation oracles. Compared to our approach, these methods demand significantly more oracle interactions and rely on oracles that are often harder to implement. Exploring offline alternatives, contextual bandit regret can be minimized with $O(\log\log T)$ offline regression-oracle calls for a known $T$ \citep{simchi2020bypassing}. Other methods employ the ``Optimism in Face of a Context'' principle, requiring $O(T)$ offline calls for bandits \citep{xu2020upper} and tabular MDPs \citep{levy2023optimism,levy2024eluder,levy2026nearoptimalregretpolicyoptimization}. Notably, the only existing offline-oracle-efficient algorithm for tabular MDPs requires just $O(H\log\log T)$ estimation calls, but still incurs a heavy planning burden of $O(HS^2A\log\log T)$ oracle calls \citep{qian2024offline}.

\begin{table}[h]
\caption{Comparison of regret guarantees and oracle-call complexities}
\label{table:comparison}
\vspace{0.3em}
\centering
\resizebox{\linewidth}{!}{%
\begin{tabular}{ccccc}
\hline
Paper & Regret & Estimation oracle type \& calls                                                                            & Planning oracle calls & Model type           \\ \hline
\citep{foster2021statistical}      & $\tilde{O}(\sqrt{T})$     & $O(T)$ online oracle calls                                                                     & $O(T)$         & general            \\
\citep{levy2023optimism}   & $\tilde{O}(T^{3/4})$  & $O(T)$ offline oracle calls                            & $O(T)$      & tabular MDP only   \\
\citep{deng2024sample}  & $\tilde{O}(\sqrt{T})$     & $O(T)$ offline oracle calls                                                                  & $O(T)$ & linear MDP \\
\citep{qian2024offline}  & $\tilde{O}(\sqrt{T})$     & \begin{tabular}[c]{@{}c@{}}$O(H\log T)$ or $O(H\log\log T)$ \\ calls to an offline oracle\end{tabular} & \begin{tabular}[c]{@{}c@{}}$O(HS^2A\log T)$ or \\ $O(HS^2A\log\log T)$ \end{tabular}        & tabular MDP only           \\ 
\citep{levy2024eluder}   & $\tilde{O}(\sqrt{T})$  & $O(T)$ offline oracle calls                            & $O(T)$      & tabular MDP only   \\
\citep{levy2026nearoptimalregretpolicyoptimization}   & $\tilde{O}(\sqrt{T})$  & $O(T)$ offline oracle calls                            & $O(T)$      & tabular MDP  only  \\
\textbf{This work}   & $\tilde{O}(\sqrt{T})$  &  \begin{tabular}[c]{@{}c@{}}$O(H\log T)$ or $O(H\log\log T)$ \\ calls to an offline oracle  \end{tabular}                            & \begin{tabular}[c]{@{}c@{}}$O(H\log T)$ or \\ $O(H\log\log T)$ \end{tabular}      & tabular MDP    \\
\textbf{This work}   & $\tilde{O}(T^{3/4})$  &  \begin{tabular}[c]{@{}c@{}}$O(H\log T)$ or $O(H\log\log T)$ \\ calls to an offline oracle  \end{tabular}                               & \begin{tabular}[c]{@{}c@{}}$O(H\log T)$ or \\ $O(H\log\log T)$ \end{tabular}      & linear MDP    \\
\hline
\end{tabular}%
}
\end{table}
In comparison, our tabular MDP algorithm achieves the optimal regret bound using only $O(H\log T)$ or $O(H\log\log T)$ calls to both the offline estimation and planning oracles, with complexity completely independent of the state and action space dimensions. This independence is a substantial improvement over existing algorithms. It also enables our method to be the first doubly oracle-efficient algorithm that attains sublinear regret in MDPs with infinite state and action spaces. Table \ref{table:comparison} summarizes the comparison.

\paragraph{Learning with log-barrier regularization} 
At a technical level, our algorithm is based on log-barrier and log-determinant regularization. Log-barrier penalties have been widely used in online and reinforcement learning to guide exploration or ensure safety \citep{foster2016learning,hazan2016introduction,zheng2019equippingexpertsbanditslongtermmemory,wei2018more,cesani2026log,usmanova2024log,ni2025safe,zhang2024constrained} and have also been adapted to adversarial settings \citep{dai2023refined,liu2023bypassing,pmlr-v178-zimmert22b}. However, their integration into offline-oracle-efficient RL remains underexplored.  There is some equivalence between log-barrier regularization and inverse gap weighting (IGW), and then the log-determinant regularization is proposed for generalizing IGW in infinite spaces \citep{foster2020adapting}. Motivated by these insights, we introduce a new integration of log-barrier regularization with a \textit{trusted occupancy measure} adapted from previous ones \citep{qian2024offline}. This combination reduces the number of planning-oracle calls and extends naturally to MDPs with infinite state and action spaces, addressing a setting that is not covered by previous ones \citep{qian2024offline}.

\paragraph{Notation.} We use $[n]$ to denote the set $\{1,2,\dots,n\}$. We adopt standard Bachmann--Landau notation $O(\cdot)$ and $\tilde{O}(\cdot)$, where $A=O(B)$ means that $A \le B$ up to a constant factor, and $A=\tilde{O}(B)$ means that $A \le B$ up to logarithmic factors. For a finite set $\cX$, we write $|\cX|$ for its cardinality. For a countable space $\cY$, $\int_{\cY}$ denotes integration over $\cY$ with respect to the counting measure. For any set $\cB$, $\Delta(\cB)$ denotes the set of all probability distributions supported on $\cB$. For a positive semidefinite matrix $U$, write $\|x\|_U^2:=x^TUx$.

\section{Tabular MDP Model Setup}\label{sec:model_setup_tabular}
We consider an episodic online reinforcement learning setting with finite horizon $H$. A tabular Markov decision process is defined by the tuple $(M,\cS,\cA,s^1)$, where $\cS$ is a finite state space with $|\cS|=S$ and $\cA$ is a finite action space with $|\cA|=A$. For trajectory-density estimation, each candidate model also specifies its conditional reward law, whose mean is $r^h$; we suppress that law in the tuple notation. A model $M$ is represented by $\cbr{P^h, r^h}_{h=1}^H$, where $P^h$ denotes the transition kernel and $r^h$ denotes the expected reward function. More specifically, the transition kernel satisfies $\PP^h(s_{h+1} = s' \mid s_h = s, a_h = a) = P^h(s' \mid s,a)$, and the expected reward at step $h$ is $r^h(s,a)$ when the agent takes action $a$ in state $s$.

The decision maker does not know the true underlying model and must learn it through sequential interaction with the environment. In episodic RL, the decision rule used throughout a trajectory is called a policy and is denoted by $\pi$. Any policy $\pi$ can be represented as a sequence $(\pi^1,\cdots,\pi^H)$, where $\pi^h:\cS\rightarrow\Delta(\cA)$ maps the state space $\cS$ to the probability simplex of the action space $\cA$. We use $\Pi_{RNS}$ to denote the set of all randomized, non-stationary Markov policies, which we call base policies.

Let $T$ denote the total number of trajectories and $M_*$ denote the true model, i.e., $M_*=\{P^h_*,r^{h}_*\}_{h\in[H]}$. The interactive protocol proceeds in $T$ rounds. In each round $t$, the $t$-th trajectory is generated by the following process:
\begin{itemize}
    \item The decision maker starts with a fixed known initial state $s_t^1=s^1\in\cS$.
    \item The decision maker chooses policy $\pi_t=(\pi_t^1,\cdots,\pi_t^H)$ to apply.
    \item For step $h=1:H$:
    \begin{itemize}
        \item The DM samples action $a_t^h$ according to the randomized policy $\pi_t^h(\cdot|s_t^h)\in\Delta(\cA)$.
        \item The DM moves to the next state $s_t^{h+1}$ according to the Markov transition kernel 
        \[
        s_t^{h+1}\sim P_*^h(\cdot|s_t^h,a_t^h),
        \]
        and receives stochastic reward $R_t^h(s_t^h,a_t^h)=r_*^h(s_t^h,a_t^h)+\xi_t^h$.
    \end{itemize}
\end{itemize}
Here, $\xi_t^h$ is zero-mean random noise independent of the DM's decisions. Throughout the paper, we assume without loss of generality that the total reward collected in each trajectory, $\sum_{h=1}^{H}R^h$, is bounded in $[0,1]$. We impose the same reward normalization on candidate models and their layerwise aggregates, so that all estimated values also belong to $[0,1]$. A common per-stage reward bound can be rescaled by $H$ to enforce this convention.

We study model-based reinforcement learning in the well-specified setting. Specifically, we assume access to a model class $\cM$ that contains $M_*$; this is the standard realizability assumption in reinforcement learning.
\begin{assumption}\label{ass:realizability_tabular}
    We have a model class $\cM$ such that the true underlying model $M_*\in\cM$.
\end{assumption}
\paragraph{Dynamic Programming.}
We now recall basic dynamic programming notation \citep{sutton1998reinforcement}. Given any model $M$, policy $\pi$, state $s$, and action $a$, we define the $Q$ function $Q_M^h(s,a;\pi) $ at layer $h$ and the value function $V_M^h(s;\pi) $ at layer $h$ as
\[
Q_M^h(s,a;\pi) =\sum_{j=h}^{H}\EE^{M,\pi}[R^j|s^h=s,a^h=a],\ \text{and}\ V^h_M(s;\pi) =\EE_{a\sim\pi^h(\cdot|s)}[Q_M^h(s,a;\pi) ].
\]
Then, we have the Bellman equation
\[
Q_M^h(s,a;\pi) =r^h_M(s,a) +\int_{s'\in\cS}P_M^h(s'|s,a) V_M^{h+1}(s';\pi) .
\]
For a base policy $\nu$, we also abbreviate $V_{M_*}^h(s;\nu)$ as $V_*^h(s;\nu)$. We denote by $\pi_{*}$ the optimal policy of the true model $M_*$, and we abbreviate its $Q$ function and value function by $Q_{M_*}^h(\cdot,\cdot) $ and $V_*^h(\cdot) $. When $h=1$, we further simplify notation by writing $V_{*}^1(s^1) $ as $V_*^1 $ and $V_{M_*}^1(s^1,\pi) $ as $V_*^1(\pi) $.

For any model $M=\{P^h,r^h\}_{h\in[H]}$ and any policy $\pi$, we let $M(\pi)$ denote the distribution of the trajectory $(s^1,a^1,r^1,\ldots,s^H,a^H,r^H)$ induced when the learner follows policy $\pi$ and the environment evolves according to model $M$. Correspondingly, we use $\PP^{M,\pi}$ and $\EE^{M,\pi}$ to denote the probability and expectation under $M(\pi)$. We also write $M(k)$ for the $k$-th layer of the model, namely $M(k)=\cbr{P^k,r^k}$ for $k=1,2,\ldots,H$. Therefore, estimating the model $M$ is equivalent to estimating each layer $M(k)$ for all $k\in[H]$.

\paragraph{Meta-policies.}
A base policy $\nu\in\Pi_{RNS}$ specifies the complete sequence of Markov action rules $(\nu^1,\ldots,\nu^H)$. A \emph{meta-policy} $p$ is a finitely supported distribution over these base policies. We use the fixed planning class
\[
\Pi_{\mathrm{mix}}
:=\Delta_{\mathrm{fin}}(\Pi_{RNS})
=\left\{\sum_{i=1}^{L}w_i\delta_{\nu_i}:
L\in\mathbb N,\ \nu_i\in\Pi_{RNS},\ w_i\ge0,\ \sum_{i=1}^{L}w_i=1\right\},
\]
where $\delta_\nu$ is the Dirac distribution at $\nu$. Executing $p=\sum_iw_i\delta_{\nu_i}$ means drawing one complete base policy $\nu_i$ with probability $w_i$ at the beginning of a trajectory and following it throughout that trajectory. For every model $M$, the trajectory law and initial value are therefore
\[
M(p)=\sum_iw_iM(\nu_i),\ 
V_M^1(p)=\sum_iw_iV_M^1(\nu_i).
\]
Accordingly, $\PP^{M,p}$ denotes this mixture trajectory law, and $\EE^{M,p}[F]=\sum_iw_i\EE^{M,\nu_i}[F]$ for any integrable trajectory function $F$.
We use $p_m^h\in\Pi_{\mathrm{mix}}$ for the meta-policy selected in segment $h$ of epoch $m$, and $\pi_t\sim p_m^h$ for the base policy independently sampled for trajectory $t$ in that segment. The class $\Pi_{\mathrm{mix}}$ is the same in every segment; the selected distribution $p_m^h$ and its finite support may change. Occupancies of a meta-policy will be defined by averaging the corresponding base-policy occupancies in Subsection~\ref{subsec:layer_construction_tabular}. Bellman identities and action recursions apply to the base policies. When a base policy is used as a meta-policy comparator, it means its Dirac distribution. In particular,
\[
\sup_{p\in\Pi_{\mathrm{mix}}}V_M^1(p)
=\sup_{\nu\in\Pi_{RNS}}V_M^1(\nu).
\]

\paragraph{Offline Density Estimation Oracles.} In the tabular MDP setting, we consider offline density-estimation oracles associated with the model class $\cM$, denoted by $\Alg$. Specifically, given any two distributions $\PP$ and $\QQ$ and some base measure $\mu$, the Hellinger divergence between them is defined as
\[
D_\mathrm{H}^2(\PP,\QQ):=\int \rbr{\sqrt{\frac{d\PP}{d\mu}}-\sqrt{\frac{d\QQ}{d\mu}}}^2d\mu.
\]
For tabular MDP model estimation, we use the following offline-regression statistical guarantee in Hellinger divergence.
\begin{definition}\label{def:offline_density_oracle}
     Given $n$ training trajectories $(\pi_i,s_i^1,a_i^1, R_i^1,\cdots,s_i^H,a_i^H,R_i^H)$, $i=1,2,\cdots,n$ such that $\cbr{\pi_i}_{i=1}^n$ are drawn i.i.d. from $p$, and $(s_i^1,a_i^1, R_i^1,\cdots,s_i^H,a_i^H,R_i^H)$ is the trajectory sampled according to $M_*(\pi_i)$. The offline density-estimation oracle $\mathtt{Alg}_{\mathtt{off}}$ returns an estimator $\hat{M}$. For any $\delta\in(0,1/2)$, with probability at least $1-\delta$, we have
    \[
    \EE_{\pi\sim p}[D_{\mathrm{H}}^2(\hat{M}(\pi),M_*(\pi))]\le \cE_{\cM}^\delta(n).
    \]
\end{definition}
When $\cM$ is finite and the density-estimation oracle $\Alg$ is maximum likelihood estimation, the offline density-estimation oracle $\Alg$ achieves the guarantee $\cE_{\cM}^\delta(n)\lesssim \frac{\log(|\cM|/\delta)}{n}$ by Lemma \ref{app-lemma:MLE_finite}. For more general model classes, $\cE_\cM^\delta(n)$ can be bounded through Dudley's entropy integral. In each segment, our algorithm fixes an episode-level mixture $p_m^h$ and draws the base policies independently from it. Thus Definition~\ref{def:offline_density_oracle} applies with $p=p_m^h$, with the sampled base-policy labels retained in the training data. We impose the following assumption on the offline density-estimation oracle $\Alg$.
\begin{assumption}
 $\cE_\cM^\delta(n)$ is decreasing in $n$ and $\delta$. Specifically, for any $n_1<n_2$, we have $\cE_\cM^\delta(n_1)>\cE_\cM^\delta(n_2)$. For any  $\delta_1<\delta_2$, we have $\cE_\cM^{\delta_1}(n)>\cE_\cM^{\delta_2}(n)$.
\end{assumption}
This assumption is natural. As the number of training samples increases, the estimation error is expected to decrease; on the other hand, a stronger confidence requirement typically results in a weaker guarantee.
\paragraph{Policy Optimization Oracles.}
In reinforcement learning, statistical estimation oracles estimate the model but do not directly specify which policies should be executed. This role is played by policy-optimization oracles, which is also called planning oracles. Given a reinforcement learning model $M$, a policy-optimization oracle typically solves the following problem:
$$p^{\mathrm{opt}}\in\argmax_{p\in\Pi_{\mathrm{mix}}}J(p;M),$$
where $J$ is an objective function parameterized by the model $M$. For example, given an estimated model $\hat{M}=(\hat{P}^h,\hat{R}^h)_{h=1}^{H}$, a UCB-type regret-minimization algorithm uses the policy-optimization oracle to maximize the estimated cumulative reward $\sum_{h=1}^{H}\hat{R}^h$ plus an exploration bonus $b_h$ \citep{auer2008near,ayoub2020model}, i.e.
\[
J(p,\hat{M})=\EE^{\hat{M},p}[\sum_{h=1}^H(\hat{R}^h(s_h,a_h)+b_h(s_h,a_h))].
\]
Meanwhile, in pure-exploration reinforcement learning \citep{jin2020reward}, the objective is, broadly speaking, to encourage the agent to visit every state-action pair at every step. This motivates the use of entropy-based objectives in policy-optimization oracles, as in maximum-entropy exploration methods \citep{hazan2019provably,zhang2021exploration}. Accordingly, at each step $h$, we adopt the Rényi entropy as the policy optimization objective $J_h(p,\hat{M})=\frac{1}{1-\alpha}\log(\sum_{s,a}(d^h_{\hat{M}}(s,a;p))^\alpha)$, where $d^h_{\hat{M}}$ is the occupancy measure at step $h$ under the estimated model $\hat{M}$. In the tabular MDP section of this paper, the objective function will be an estimated value function with a log-barrier regularization term. We will specify the objective function in  Subsection \ref{subsec:layer_construction_tabular}.

The regret of a base policy $\nu$ is $\mathrm{reg}(\nu):=V_*^1-V_*^1(\nu)$. For a meta-policy $p=\sum_iw_i\delta_{\nu_i}$, linearity of the initial value gives
\[
\mathrm{reg}(p):=V_*^1-V_*^1(p)=\sum_iw_i\mathrm{reg}(\nu_i).
\]
The total expected regret of the decision maker is defined by
\[
\mathrm{Reg}(T):=\sum_{t=1}^{T}\EE_t[\mathrm{reg}(\pi_t)],
\]
where $\cH_t$ is the history before the base-policy draw in round $t$, and $\EE_t[\cdot]=\EE[\cdot\mid\cH_t]$. In segment $(m,h)$, $\EE_t[\mathrm{reg}(\pi_t)]=\mathrm{reg}(p_m^h)$.

\section{Double-Oracle-Efficient Reinforcement Learning: Tabular MDP}\label{sec:tabular_alg_design}
In this section, we present our doubly oracle-efficient reinforcement learning algorithm for tabular MDPs. Subsection \ref{subsec:alg_design_tabular} gives an overview and explains its efficiency with respect to both the offline statistical density-estimation oracle and the policy-optimization oracle. Subsection \ref{subsec:layer_construction_tabular} then provides the detailed construction and the main ideas behind the algorithm.
\subsection{Algorithm Design}\label{subsec:alg_design_tabular}
The algorithm proceeds as follows. We divide the $T$ rounds of interaction into $N$ epochs. An epoch schedule is a finite sequence $0=\tau_0<\tau_1<\cdots<\tau_N=\frac{T}{H}$; the specific values of $\tau_1,\tau_2,\cdots,\tau_N$ are specified later. For any round $t$, let $m(t)$ denote its epoch. Epoch $m$ contains $H(\tau_m-\tau_{m-1})$ rounds and is evenly divided into $H$ segments, each of length $\tau_m-\tau_{m-1}$. In the $h$-th segment of epoch $m$, the learner fixes an episode-level mixture $p^h_m\in\Pi_{\mathrm{mix}}$ at the beginning of the segment. Independently for each trajectory in this segment, the decision maker draws $\pi_t\sim p^h_m$ and executes the sampled base policy for the entire trajectory. After collecting the trajectory data from the $h$-th segment of the $m$-th epoch, $\cbr{\pi_t}\cup\{s_t^j,a_t^j,R_t^j\}_{j=1}^{H}$ for $\tau_{m-1}H+(\tau_m-\tau_{m-1})(h-1)+1\le t\le \tau_{m-1}H+(\tau_m-\tau_{m-1})h$, we call the offline regression oracle $\Alg$ with this data and obtain an estimated model $\hat{M}^h_m=\cbr{(\hat{P}^j_{m,h},\hat{r}^ {j}_{m,h})_{j=1}^{H}} $. From this model, we use only the $h$-th layer: the estimated transition kernel $\hat{P}^h_{m,h}$ from step $h$ to step $h+1$ and the estimated reward function $\hat{r}^{h}_{m,h}$ at step $h$. We then aggregate the layer-wise estimates into a single model used in the next epoch $m+1$.

For the aggregation step, note that at any fixed epoch $m$ there are $H$ estimated models
\[
\hat{M}_m^h= \left\{(\hat{P}^j_{m,h},\hat{r}^ {j}_{m,h})_{j=1}^{H}\right\},\ h=1,2,\cdots,H.
\]
For each model $\hat{M}_m^h$, we use only its output at layer $h$. Specifically, we extract the estimated transition kernel $\hat{P}^h_{m,h} $
and the estimated reward function $\hat{r}^{h}_{m,h} $
at step $h$. Repeating this procedure for every layer $h\in[H]$ and combining the corresponding outputs across all $H$ layers gives the aggregated estimated model produced by epoch $m$, denoted by $\hat{M}_{m}^{out}$:
\[
\hat{M}_{m}^{out}
=
\left\{(\hat{P}^h_{m,h},\hat{r}^h_{m,h})_{h=1}^{H}\right\}.
\]
We now introduce notation used in the algorithm and analysis. For a base policy $\nu$, let $\hat V_m^h(s;\nu):=V_{\hat M_m^{out}}^h(s;\nu)$ and $\hat Q_m^h(s,a;\nu):=Q_{\hat M_m^{out}}^h(s,a;\nu)$. For a meta-policy $p=\sum_iw_i\delta_{\nu_i}$, write $\hat V_m^1(p):=\sum_iw_i\hat V_m^1(\nu_i)$. Let
\[
p_m\in\argmax_{p\in\Pi_{\mathrm{mix}}}\hat V_m^1(p),\ 
\hatreg_m(p):=\hat V_m^1(p_m)-\hat V_m^1(p).
\]
The benchmark $p_m$ may be chosen as a Dirac distribution at an optimal base Markov policy of $\hat M_m^{out}$; it is distinct from the segment meta-policy $p_m^h$. For a base policy $\nu$, $\hatreg_m(\nu)$ abbreviates $\hatreg_m(\delta_\nu)$. Finally, set $\hat P_m^h:=\hat P_{m,h}^h$ and $\hat r_m^h:=\hat r_{m,h}^h$ for the layer-$h$ kernel and reward of $\hat M_m^{out}$. These conventions also apply to the initial model $\hat M_0^{out}$.

Figure \ref{fig:model-update} illustrates the algorithmic structure within each epoch.
\begin{figure}[h]
\centering
\begin{tikzpicture}[
    >=stealth,
    shorten >=2pt,
    shorten <=2pt,
    every node/.style={inner sep=1pt},
    main/.style={font=\normalsize}
]

\node[main] (prev) at (-1.3,1.2) {$\left\{\hat{P}^h_{m-1,h},\hat{r}^h_{m-1,h} \right\}_{h\in[H]}$};

\node[main] (pi1)  at (2.0,1.2) {$p^1_m$};
\node[main] (ph1)  at (2.0,0.0) {$\hat{P}^1_{m,1},\hat{r}^1_{m,1}$};
\node[main] (td1)  at (4.5,0.0) {$\tilde{\cT}_m^2,\tilde{d}_m^2$};

\node[main] (pi2)  at (7.0,1.2) {$p^2_m$};
\node[main] (ph2)  at (7.0,0.0) {$\hat{P}^2_{m,2},\hat{r}^2_{m,2}$};
\node[main] (td2)  at (9.0,0.0) {$\tilde \cT_m^3,\tilde d_m^3$};

\node[main] (dots) at (11.0,1.2) {$\cdots$};
\node[main] (phH)  at (11.0,0.0) {$\hat{P}^H_{m,H},\hat{r}^H_{m,H}$};

\draw[->] (prev) -- (pi1);
\draw[->] (prev) to[out=8,in=172] (pi2);
\draw[->] (prev) to[out=12,in=172] (dots);

\draw[->] (pi1) -- (ph1);
\draw[->] (ph1) -- (td1);
\draw[->] (td1) to[out=20,in=220] (pi2);

\draw[->] (pi2) -- (ph2);
\draw[->] (ph2) -- (td2);
\draw[->] (td2) to[out=20,in=220] (dots);

\draw[->] (dots) -- (phH);
\end{tikzpicture}
\caption{Dependency graph of Algorithm \ref{alg:offline_tabular_RL} in epoch $m$. The estimated model $\hat{M}_{m-1}^{out}=\{\hat{P}^h_{m-1,h},\hat{r}^h_{m-1,h}\}_{h\in[H]}$ is used to compute the estimated value function $\hat{V}^1_{m-1}$ when constructing $p^h_m$. At the end of segment $h$, the transition and reward estimates $\hat{P}^h_{m,h},\hat{r}^h_{m,h}$ are generated by calling the offline regression oracle $\Alg$. We then compute the trusted transition set $\tilde{\cT}^{h+1}_m$ and the trusted occupancy measure $\tilde{d}^{h+1}_m$ using Definition \ref{def:trusted_occupancy_measure}.}
\label{fig:model-update}
\end{figure}

With these ingredients in place, Algorithm \ref{alg:offline_tabular_RL} gives the pseudocode for \texttt{DOERL}.
\begin{algorithm}[h]
\begin{algorithmic}[1]
\caption{Doubly Oracle-Efficient Reinforcement Learning (\texttt{DOERL}): Tabular MDP}\label{alg:offline_tabular_RL}
\REQUIRE epoch schedule $0=\tau_0<\tau_1<\cdots<\tau_N=T/H$, confidence parameter $0<\delta<1/2$, model class $\cM$, offline regression oracle $\Alg$.
\STATE Initialize: $\hat{M}_0^{out}$ to be any normalized admissible tabular MDP model with state space $\cS$ and action space $\cA$.
\FOR{epoch $m=1,2,\cdots,N$}
\STATE Set the confidence level $\delta_m$, error bound $\cE_m$, and hyperparameters $\beta_m,\zeta_m,\eta_m$ as in Theorem \ref{thm:main_regret_tabular}.
 \FOR{segment $h=1,\cdots,H$}
 \STATE {Compute
 \[
p^h_m\in\argmax_{p\in\Pi_{\mathrm{mix}}}\cbr{\hat{V}_{m-1}^1(p)+\frac{1}{\eta_m}\sum_{s,a}\log(\tilde{d}^h_m(s,a;p)+\beta_m)},
 \]
 where $\tilde{d}^h_m(s,a;p)$ is defined in Subsection \ref{subsec:layer_construction_tabular}.}
 \FOR{\parbox[t]{.75\linewidth}{$t=\tau_{m-1}H+(\tau_m-\tau_{m-1})(h-1)+1,\cdots,\tau_{m-1}H+(\tau_m-\tau_{m-1})h$}}
 
 \STATE Independently draw $\pi_t\sim p^h_m$ and execute $\pi_t$ for the entire trajectory.
 \STATE Observe the trajectory $\cbr{\pi_t,s_t^1,a_t^1,R_t^1,\cdots,s_t^H,a_t^H,R_t^H}$.
 
 \ENDFOR
\STATE Run the offline density-estimation oracle $\Alg$ at confidence $\delta_m$, with the input trajectory data $\cbr{\pi_t,\{s_t^j,a_t^j,R_t^j\}_{j\in[H]}\big|\tau_{m-1}H+(\tau_m-\tau_{m-1})(h-1)< t\le\tau_{m-1}H+(\tau_m-\tau_{m-1})h]}$.
\STATE Obtain the estimated model $\hat{M}^h_m=\cbr{(\hat{P}^j_{m,h}, \hat{r}^ {j}_{m,h})_{j=1}^{H}} $ 
 \ENDFOR
 \STATE Formulate the model $\hat{M}_m^{out}=\cbr{(\hat{P}^h_{m,h},\hat{r}^ {h}_{m,h})_{h=1}^{H}}$.
\ENDFOR
\end{algorithmic}
\end{algorithm}
\subsection{Per-Epoch Detailed Construction: Tabular MDP}\label{subsec:layer_construction_tabular}
Our primary technical contribution in this section is the development of a novel \emph{truncated trusted occupancy measure}. Significantly advancing the prior framework of \citep{qian2024offline}, our refined construction introduces a trusted transition set that systematically filters out low-frequency state-action-transition triples.

We first recall the occupancy measure in reinforcement learning.
\begin{definition}\label{def:occupancy_measure}
For a model $M$, a base policy $\nu\in\Pi_{RNS}$, and a layer $h$, the occupancy measure is
\[
d_M^h(s,a;\nu):=\PP^{M,\nu}(s^h=s,a^h=a),\ 
d_M^h(s;\nu):=\sum_a d_M^h(s,a;\nu).
\]
\end{definition}
For $p=\sum_iw_i\delta_{\nu_i}\in\Pi_{\mathrm{mix}}$, the same trajectory experiment gives
\[
d_M^h(s,a;p)=\sum_iw_i d_M^h(s,a;\nu_i),\ 
d_M^h(s;p)=\sum_iw_i d_M^h(s;\nu_i).
\]

\begin{definition}\label{def:trusted_occupancy_measure}
Fix an epoch $m$. For each base policy $\nu\in\Pi_{RNS}$, initialize
\[
\tilde d_m^1(s;\nu)=\II(s=s^1),\ 
\tilde d_m^1(s,a;\nu)=\II(s=s^1)\nu^1(a\mid s).
\]
For $h=1,\ldots,H-1$, given the common trusted transition set $\tilde{\cT}_m^{h+1}$, define
\[
\tilde d_m^{h+1}(s';\nu)
=\sum_{(s,a):(s,a,s')\in\tilde{\cT}_m^{h+1}}
\tilde d_m^h(s,a;\nu)\hat P_m^h(s'\mid s,a),
\]
\[
\tilde d_m^{h+1}(s',a';\nu)
=\tilde d_m^{h+1}(s';\nu)\nu^{h+1}(a'\mid s').
\]
For a meta-policy $p=\sum_iw_i\delta_{\nu_i}$, set
\[
\tilde d_m^h(s,a;p):=\sum_iw_i\tilde d_m^h(s,a;\nu_i),\ 
\tilde d_m^h(s;p):=\sum_iw_i\tilde d_m^h(s;\nu_i).
\]
After executing $p_m^h$ in segment $h$ and estimating $\hat P_m^h$, construct
\[
\tilde{\cT}_m^{h+1}
:=\left\{(s,a,s'):
\tilde d_m^h(s,a;p_m^h)\hat P_m^h(s'\mid s,a)\ge\zeta_m^{-1}\right\}.
\]
This trusted set is shared by every base policy and every meta-policy component. The threshold parameter $\zeta_m>0$ is specified in Theorem~\ref{thm:main_regret_tabular}.
\end{definition}
At the beginning of segment $h$, the sets $\tilde{\cT}_m^2,\ldots,\tilde{\cT}_m^h$ and kernels $\hat P_m^1,\ldots,\hat P_m^{h-1}$ have already been computed. They are held fixed when the planning oracle varies $p\in\Pi_{\mathrm{mix}}$ in the segment-$h$ objective. Only after that segment produces its data and $\hat P_m^h$ do we form $\tilde{\cT}_m^{h+1}$. Thus the occupancy used to select $p_m^h$ depends on the completed segments of the current epoch, while $\hat V_{m-1}^1$ uses the previous epoch's model.

For a base policy $\nu$, define the \textbf{true observable occupancy} by the same recursion and trusted sets, with $P_*^h$ replacing $\hat P_m^h$:
\[
d_m^1(s;\nu)=\II(s=s^1),\  d_m^1(s,a;\nu)=\II(s=s^1)\nu^1(a\mid s),
\]
\[
d_m^{h+1}(s';\nu)
=\sum_{(s,a):(s,a,s')\in\tilde{\cT}_m^{h+1}}d_m^h(s,a;\nu)P_*^h(s'\mid s,a),
\]
\[
d_m^{h+1}(s',a';\nu)=d_m^{h+1}(s';\nu)\nu^{h+1}(a'\mid s').
\]
The \textbf{estimated occupancy} is the full occupancy of the aggregated model:
\[
\hat d_m^h(s,a;\nu):=d_{\hat M_m^{out}}^h(s,a;\nu),\ 
\hat d_m^h(s;\nu):=d_{\hat M_m^{out}}^h(s;\nu).
\]
As above, for $p=\sum_iw_i\delta_{\nu_i}$, define
\[
d_m^h(s,a;p):=\sum_iw_i d_m^h(s,a;\nu_i),\ 
\hat d_m^h(s,a;p):=\sum_iw_i\hat d_m^h(s,a;\nu_i),
\]
with the same averages for state occupancies. \emph{In this paper, we use $p$ to denote the meta policy. $\pi$, $\rho$, $\nu$ are used to denote base Markovian policies.}

Thus, $d_{M_*}^h$ and $\hat d_m^h$ are full occupancies. In contrast, $d_m^h$ and $\tilde d_m^h$ retain only trajectories whose first $h-1$ transitions lie in the common trusted sets, using the true and estimated kernels respectively. They can have total mass below one, and
\[
0\le d_m^h(s,a;p)\le d_{M_*}^h(s,a;p),\ 
0\le\tilde d_m^h(s,a;p)\le\hat d_m^h(s,a;p).
\]
All these quantities refer to a fresh trajectory with the learned models and trusted sets held fixed. Each occupancy is affine in $p$ because the base-policy quantities in its average are fixed.

In Algorithms \ref{alg:offline_tabular_RL} and \ref{alg:offline_linear_RL}, one planning-oracle call optimizes the displayed objective over $\Pi_{\mathrm{mix}}$ and returns the complete finite mixture. As with the original argmax formulation, attainment of the requested optimum is part of the oracle specification. For any attained moment vector, Carath\'eodory's theorem permits a representation with at most $SA+2$ base policies in the tabular case, or $2+h d(d+1)/2$ base policies for the linear objective in segment $h$. These are support bounds, not additional planning-oracle calls.

For Algorithm \ref{alg:offline_tabular_RL}, the policy-optimization oracle and the offline density-estimation oracle are therefore each called $NH$ times, where $N$ is the number of epochs. In Section \ref{sec:theory_tabular}, we show that an appropriate epoch schedule gives $N\asymp \log\log T$, so the algorithm requires only $O(H\log \log T)$ oracle calls.

Crucially, our construction drastically simplifies implementation compared to prior designs \citep{qian2024offline}. While previous methods determine tuple inclusion via $\{(s,a,s') \mid \max_{p\in\Pi_{\mathrm{mix}}}\{\frac{\tilde{d}^h_m(s,a;p)}{SA}\cdot\hat{P}^h_m(s'|s,a)\}\ge \frac{1}{\zeta_m}\}$, they necessitate a prohibitive optimization over the entire policy space. In contrast, our approach completely bypasses this computational bottleneck. We define the trusted transition set as $\cbr{(s,a,s') \mid \tilde{d}^h_m(s,a;p^h_m) \cdot\hat{P}^h_m(s'|s,a)\ge \frac{1}{\zeta_m}}$. By directly leveraging the previously executed meta-policy $p^h_m$, our formulation entirely eliminates the need for global policy optimization, yielding a highly efficient and directly computable criterion.

The objective function that our planning oracle optimizes at the beginning of segment $h$ in epoch $m$:
$$J^h_m(p):=\hat{V}^1_{m-1}(p)+\frac{1}{\eta_m}\sum_{s,a}\log(\tilde{d}^h_m(s,a;p)+\beta_m),$$
can be viewed as the estimated value function with a log-barrier regularization \citep{foster2016learning} weighted by the parameter $\eta_m$. Specifically, the first term in the objective, $\hat{V}_{m-1}^1(p)$, captures exploitation since maximizing this value function amounts to maximizing the total expected reward along the trajectory. In contrast, the log-barrier term encourages exploration \citep{foster2020adapting}: maximizing this quantity is closely related to maximizing the entropy \citep{zhang2021exploration}, thereby promoting broader exploration of the environment \citep{dai2023refined,jin2023no}. Therefore, by tuning the weight parameter $\eta_m$ properly, the policy achieves a suitable exploration-exploitation trade-off and thus controls the regret. 

Moreover, our log-barrier regularization admits a natural interpretation as a generalization of the well-known inverse-gap-weighting policy \citep{foster2020adapting}. This connection clarifies how the proposed method adaptively balances exploration and exploitation while remaining compatible with low-frequency oracle-call frameworks \citep{simchi2020bypassing}. Computationally, the resulting objective takes the form of a structured log-barrier optimization problem, which can be solved efficiently using interior-point methods \citep{liu2020ipo,usmanova2024log}. This places our approach within a mature algorithmic toolkit developed for linear programming and related convex optimization problems \citep{lee2025analysis}.

\section{Theoretical Guarantees: Tabular MDP}\label{sec:theory_tabular}
In this section, we establish the theoretical regret guarantees for our tabular MDP framework. Subsection \ref{subsec:theory_perepoch_tabular} derives per-epoch bounds that control the value-function estimation error via pseudo-regret. To support this, we introduce several key lemmas, culminating in the proof sketch provided for Lemma \ref{lemma:perepoch_valuefunc_error_tabular}. Subsequently, Subsection \ref{subsec:theory_regret_analysis_tabular} details the hyperparameter configuration for Algorithm \ref{alg:offline_tabular_RL}, specifically, $\tau_m, \zeta_m, \beta_m$, and $\eta_m$ for epochs $m=1,2,\dots$. It also aggregates the per-epoch guarantees to establish our overarching regret bound.
\subsection{Per-Epoch Theoretical Analysis}\label{subsec:theory_perepoch_tabular}
In this subsection, we bound the value function estimation error between $M_*$ and the estimated model $\hat{M}_m^{\mathrm{out}}$ using the pseudo-regret evaluated under the previous model, $\hat{M}_{m-1}^{\mathrm{out}}$. To establish such results, we need intermediate lemmas that connect the trusted occupancy measure with the estimated occupancy measure and the true observable occupancy measure.

First, we show that our selected meta-policy $p^h_m$ has a coverage guarantee in terms of the trusted occupancy measure.
\begin{lemma}\label{lemma:hatpi_property_tabular}
    For any $m,h,s,a$ and $p'\in\Pi_{\mathrm{mix}}$, we have
    \[
    \tilde{d}^h_m(s,a;p')\le (\tilde{d}^h_m(s,a;p^h_m)+\beta_m)(SA+\eta_m\hatreg_{m-1}(p')).
    \]
\end{lemma}
Lemma \ref{lemma:hatpi_property_tabular} shows that the trusted occupancy measure of any meta-policy $p'$ is controlled by that of the selected meta-policy $p^h_m$, up to additional small terms. Intuitively, if we can establish that $\tilde{d}^h_m$ is close to the estimated occupancy measure, then Lemma \ref{lemma:hatpi_property_tabular} implies that the executed policy satisfies some guarantees similar to the maximum ``entropy'' exploration property. This underscores why a single meta-policy $p^h_m$ suffices to construct the trusted occupancy measure, successfully bypassing the complexities of prior approaches \citep{qian2024offline}. Consequently, to rigorously bound the discrepancy between $\tilde{d}^h_m$ and $\hat{d}^h_m$, we establish the following lemma.
\begin{lemma}\label{lemma:tilded_boundby_hatd_tabular}
    For any $m,h,s,a$ and $p\in\Pi_{\mathrm{mix}}$, we have
    \[
    \hat{d}^h_m(s,a;p)-\tilde{d}^h_m(s,a;p)\le \frac{HS^3A^3}{\zeta_m}+\frac{\eta_mHS^2A^2\hatreg_{m-1}(p)}{\zeta_m}+HS^2A^2(1+\eta_m\hatreg_{m-1}(p))\beta_m.
    \]
\end{lemma}
While Lemma \ref{lemma:tilded_boundby_hatd_tabular} effectively connects the trusted occupancy measure and the estimated occupancy measures, bounding the value function estimation error further requires controlling the deviation between $\tilde{d}^h_m$ and the true observable occupancy measure $d^h_m$. The following lemma formalizes this crucial bound.
\begin{lemma}\label{lemma:tilde_dP<dP*_tabular}
Fix an epoch $m$ and the common trusted sets from
Definition~\ref{def:trusted_occupancy_measure}. Suppose that, for every
$h\in[H-1]$, the executed meta-policy $p_m^h$ satisfies
\[
\sum_{s,a}d_m^h(s,a;p_m^h)
D_H^2\bigl(\hat P_m^h(\cdot\mid s,a),P_*^h(\cdot\mid s,a)\bigr)
\le\cE_m.
\]
If $e^2(H+1)^2\cE_m<\zeta_m^{-1}$, then for every
$p\in\Pi_{\mathrm{mix}}$, $h\in[H]$, $s\in\cS$, and $a\in\cA$,
\[
\tilde d_m^h(s,a;p)\le(1+1/H)^{2h}d_m^h(s,a;p).
\]
The same conclusion holds for each base policy by taking $p=\delta_\nu$.
\end{lemma}
Utilizing Lemma \ref{lemma:hatpi_property_tabular}, \ref{lemma:tilded_boundby_hatd_tabular}, and \ref{lemma:tilde_dP<dP*_tabular}, we obtain the following per-epoch value function estimation error bound, which is summarized in Lemma \ref{lemma:perepoch_valuefunc_error_tabular}.
\begin{lemma}\label{lemma:perepoch_valuefunc_error_tabular}
  Suppose that $e^2(H+1)^2\cE_m<\zeta_m^{-1}$.  On the layerwise Hellinger-error event with bound $\cE_m$ for transitions at $h<H$ and rewards at $h\le H$, for every $p\in\Pi_{\mathrm{mix}}$, we have
    \begin{align*}
    |\hat{V}^1_m(p)-V_*^1(p)|\le& \frac{2H^2S^4A^4}{\zeta_m}+\frac{2\eta_mH^2S^3A^3\hatreg_{m-1}(p)}{\zeta_m}+2H^2S^3A^3(1+\eta_m\hatreg_{m-1}(p))\beta_m\\
+&4e^2H^2SA(SA+\eta_m\hatreg_{m-1}(p))\beta_m+4e^2H(SA+\eta_m\hatreg_{m-1}(p))\sqrt{\cE_m}.
    \end{align*}
\end{lemma}
\paragraph{Proof Sketch of Lemma \ref{lemma:perepoch_valuefunc_error_tabular}}
Apply the simulation lemma componentwise and average. The omitted occupancy mass is controlled by Lemma \ref{lemma:tilded_boundby_hatd_tabular}; on retained trajectories, combine the coverage bound in Lemma \ref{lemma:hatpi_property_tabular}, the policy-independent transition comparison in Lemma \ref{lemma:tilde_dP<dP*_tabular}, and Cauchy--Schwarz for the layerwise Hellinger errors. The complete calculation appears in the Appendix.

Finally, the next lemma shows that the pseudo-regret of $p^h_m$ is also well controlled. Thus, Lemma \ref{lemma:hatpi_property_tabular}, \ref{lemma:tilded_boundby_hatd_tabular}, and this lemma together illustrate that our executed policy achieves appropriate exploration-exploitation trade-off.
\begin{lemma}\label{lemma:pseudo_regret_tabular}
    For any $m,h$, we have \[
    \hatreg_{m-1}(p^h_m)\le \frac{SA}{\eta_m}\log(1+\frac{1}{SA\beta_m}).
    \]
\end{lemma}
 
\subsection{Regret Analysis}\label{subsec:theory_regret_analysis_tabular}
In this subsection, we establish the formal regret guarantee. We begin by delineating the valid hyperparameter configurations. For concreteness, we state the result for finite model classes, where $\cE_\cM^\delta(n)\lesssim \frac{\log(|\cM|/\delta)}{n}$. For infinite model classes, $\cE_\cM^\delta(n)$ can be readily upper bounded either via covering number arguments based on the structure of $\cM$ \citep{vershynin2018high} or through recently advanced black-box evaluation procedures \citep{wainwright2025wild,hu2026interleavedresamplingrefittingdata}.

\begin{theorem}\label{thm:main_regret_tabular}
For each epoch write $n_m=\tau_m-\tau_{m-1}$. If the number $N$ of epochs is known, set $\delta_m=\delta/[4H(N+1)^2]$; otherwise set $\delta_m=\delta/[4H(m+1)^2]$. Let $\cE_m>0$ be a deterministic nonincreasing upper bound on $4\cE_{\cM}^{\delta_m}(n_m)$. Set
\[
C_{\rm tab}=10000(H+1)^5S^5A^5,\ 
a_0=1,\  a_m=\min\{1,C_{\rm tab}\sqrt{\cE_m}\},
\]
\[
\beta_m=\min\{1,\cE_m\},\ 
\eta_m=\frac1{100a_{m-1}},\ 
\frac1{\zeta_m}=2e^2(H+1)^2\cE_m+10(H+1)SA\eta_m\cE_m,
\]
and let $L_m=1+\log(1+1/(SA\beta_m))$. Then, with probability at least $1-\delta$,
\[
\mathrm{Reg}(T)
\le H\sum_{m=1}^N n_m\min\{1,210SA L_m a_{m-1}\}.
\]
The same bound applies to an incomplete final epoch with its actual number of executed trajectories. The sequence $\cE_m$ may be any nonincreasing majorant of the stated oracle bounds; for finite model classes and the schedules below it has order $\log(|\cM|/\delta_m)/n_m$.

\end{theorem}

\paragraph{Proof Sketch of Theorem \ref{thm:main_regret_tabular}}
The layerwise Hellinger errors are bounded by four times the trajectory error, using the hybrid law $P_XQ_{Y\mid X}$. The chosen threshold makes Lemma \ref{lemma:tilde_dP<dP*_tabular} valid in every epoch. Substitution into Lemma \ref{lemma:perepoch_valuefunc_error_tabular} gives
\[
|\hat V_m^1(p)-V_*^1(p)|\le a_m(1+\eta_m\hatreg_{m-1}(p)).
\]
The lagged choice $\eta_m=(100a_{m-1})^{-1}$ and Lemma \ref{app-lemma:lagged_calibration} imply $\mathrm{reg}(p_m^h)\le2\hatreg_{m-1}(p_m^h)+10a_{m-1}$. Lemma \ref{lemma:pseudo_regret_tabular} then bounds this by $210SA L_ma_{m-1}$. Sum over the trajectories of each segment.

When the model class is finite or parametric and the maximum likelihood estimator is used as the offline density-estimation oracle, an appropriate epoch schedule $\cbr{\tau_m}_{m=1}^{N}$ yields optimal $\tilde{O}(\sqrt{T})$ regret. The algorithm requires only $O(H\log\log T)$ calls to the offline density-estimation oracle and the policy-optimization oracle when $T$ is known, and only $O(H\log T)$ such calls when $T$ is unknown.
\begin{corollary}\label{cor:reg_known_T_tabular}
Let $B=T/H\ge16$ be known. Use the increasing endpoints $\tau_m=2B^{1-2^{-m}}$ while they are at most $B/2$, and then append the endpoint $\tau_N=B$, with integer rounding. Thus $N=O(\log\log(B+16))$. Use maximum likelihood estimation and the parameters of Theorem \ref{thm:main_regret_tabular}, with a nonincreasing bound $\cE_m=80\log(|\cM|/\delta_m)/n_m$. Then, with probability at least $1-\delta$,
\[
\mathrm{Reg}(T)
\lesssim C_{\rm tab}SA\sqrt H\,[1+\log(1+T)]
\sqrt{\log\!\left(\frac{4H|\cM|(N+1)^2}{\delta}\right)}\,N\sqrt T.
\]
In particular the regret is $\tilde O(\sqrt T)$, and both oracle-call counts are $NH=O(H\log\log T)$. For $B<16$, the trivial bound $\mathrm{Reg}(T)\le T$ is absorbed by the displayed bound after enlarging its absolute constant.

\end{corollary}

\begin{corollary}\label{cor:reg_unknown_T_tabular}
If $T$ is unknown, use $\tau_m=2^m$, $m\ge1$, maximum likelihood estimation and the anytime confidence levels of Theorem \ref{thm:main_regret_tabular}. Choose a nonincreasing majorant $\cE_m\asymp\log(|\cM|/\delta_m)/n_m$. Then, for every fixed evaluation horizon $T$, with probability at least $1-\delta$,
\[
\mathrm{Reg}(T)
\lesssim C_{\rm tab}SA\sqrt H\,[1+\log(1+T)]
\sqrt{\log\!\left(\frac{4H|\cM|(N+2)^2}{\delta}\right)}\,\sqrt T,
\]
where $N=O(\log(T+2))$ is the number of started epochs. Both oracle-call counts are $O(H\log T)$, and the regret is $\tilde O(\sqrt T)$.

\end{corollary}
\paragraph{Summary.} This concludes our analysis of tabular MDPs with finite state-action spaces. For this regime, we introduced an algorithm that attains $\tilde{O}(T^{1/2})$ regret while requiring only sparse invocations of both the policy-optimization and offline density estimation oracles. We now transition to linear MDPs, characterized by infinite state and action spaces. To tackle this setting, we present an algorithm that elegantly reduces online regret minimization to statistical offline regression. Crucially, this reduction is achieved by coupling a modified trusted occupancy measure with a novel log-determinant regularization term to systematically guide policy selection.
\section{Linear MDP Model Setup}\label{sec:model_setup_linear}
\subsection{Motivation and Technical Difficulty}
In tabular MDPs, regret bounds inherently scale with the cardinalities of the state and action spaces, manifesting the classic ``curse of dimensionality'' \citep{bellman1966dynamic}. However, practical applications frequently involve infinite state spaces and highly unstructured action spaces, rendering online estimation oracles prohibitively difficult to construct. Consequently, designing online regret-minimization algorithms that maintain dual efficiency with respect to both offline estimation and policy optimization oracles, poses a formidable challenge. Transcending the tabular regime, we investigate linear MDPs with infinite state and action spaces. Crucially, we propose the first algorithm to achieve sublinear regret in this setting while requiring only sparse invocations of an offline regression oracle and a policy-optimization oracle.

Technically, offline oracle-efficient learning for linear MDPs is substantially more difficult, and the previous trusted occupancy measure construction no longer applies directly. The key difficulty is that the state and action spaces are now infinite. Unlike the tabular setting, where state-action pairs are finite, linear MDPs with infinite state-action spaces do not allow us to check every $(s,a,s')$ tuple or establish entry-wise guarantees for each transition. Instead, we need aggregate statistical guarantees obtained by integrating over newly defined trusted occupancy measures. To address this challenge, we introduce a new trusted transition set and trusted occupancy measure that can handle the infinite-dimensional state and action structure of linear MDPs.

We present our results for doubly oracle-efficient reinforcement learning in linear MDPs as follows. Subsection \ref{subsec:model_linearMDP} introduces the model setup, notation, and assumptions. Section \ref{sec:linear_alg_design} presents the algorithm and its detailed construction. Section \ref{sec:theory_linear} provides the theoretical guarantees and regret analysis.
\subsection{Linear MDP}\label{subsec:model_linearMDP}
A linear Markov decision process is defined by the tuple $(M,\cS,\cA,s^1)$, where $\cS$ is the state space, $\cA$ is the action space, and $s^1\in\cS$ is a fixed starting state. $M$ denotes the decision-making model. Specifically, for any model $M$, the adaptive decision-making process is represented as $M=\cbr{\phi_{h},\mu^{h+1},r^{h}}_{h=1}^{H}$, where $\phi_h:\cS\times\cA\rightarrow\RR^d$ is a feature map and $\mu^{h+1}:\cS\rightarrow\RR^d$ is a coefficient-vector mapping. Here, $d>1$ is the feature dimension. In the linear MDP model $M$, at step $h\in[H]$ of any trajectory, if the learner takes action $a_h=a$ in state $s_h=s$, then the transition probability of moving to $s_{h+1}=s'$ is characterized by the following Euclidean inner product in $\RR^d$:
\[
\PP^{h}(s_{h+1}=s'|s_h=s,a_h=a) =\inner{\phi_h(s,a)}{\mu^{h+1}(s')}.
\]
The function $r^{h}:\cS\times\cA\rightarrow \RR^+$ is the expected reward function of model $M$: the DM receives a stochastic reward with expectation $r^{h}(s,a)$ after taking action $a$ in state $s$ at step $h\in[H]$ under model $M$. In the linear MDP model $M$, the expected reward is also parameterized linearly in the feature map $\phi_h$, namely $r^h(s,a)=\inner{\phi_h(s,a)}{\theta^h}$ for some unknown coefficient vector $\theta^h \in \RR^d$. We use the collections $\cbr{\phi_h,\mu^{h+1},r^h}_{h=1}^{H}$ and $\cbr{\phi_h,\mu^{h+1},\theta^h}_{h=1}^{H}$ interchangeably to denote the linear MDP model $M$. \emph{In this paper, we assume that the state space $\cS$ is countably infinite and the action space $\cA$ is arbitrary.}

For a non-discrete action space, for a Markovian policy $\rho$, $d_m^h(s,a;\rho)$ and the analogous full, estimated, and trusted occupancies are understood as state-action measures. For example,
\[
\int_{s,a}f(s,a)d_m^h(s,a;\rho)
:=\sum_{s\in\cS}d_m^h(s;\rho)\int_{\cA}f(s,a)\rho^h(da\mid s).
\]
The corresponding measure for a meta-policy is the weighted average of its base-policy measures. State-action occupancy inequalities in this section mean measure domination; the displayed base-policy action recursions use the same action kernel on both sides.

In linear MDPs, the feature maps $\cbr{\phi_h}_{h\in[H]}$ are known to the decision maker, whereas the coefficient vectors $\{\mu^{h+1}_*(\cdot),\theta_*^h\}_{h\in[H]}$ are unknown and must be estimated. As in Section \ref{sec:model_setup_tabular}, we assume access to a model class $\cM$ that contains the true model $M_*$.
\begin{assumption}\label{ass:realizability_linear}
    We have a model class $\cM$ such that the true underlying model $M_*\in\cM$.
\end{assumption}
Beyond realizability, we impose a normalization condition on the coefficient-vector mappings \citep{mhammedi2023efficient,modi2024model}.
\begin{definition}\label{def:alpha_normal}
    For a linear MDP model $M=\cbr{(\phi_h,\mu^{h+1},\theta^h)_{h=1}^{H}} $, we say that $M$ satisfies the normalization condition if there exists $c_M<\infty$ such that for any $h\in[H]$,
    \[
    \int_{s'\in\cS}\|\mu^{h+1}(s')\|^{1/2}\le c_{M}<\infty,\  \|\theta^h\|_2\le 1.
    \]
\end{definition}

For a finite tabular MDP represented with one-hot features, Definition \ref{def:alpha_normal} holds with a finite constant depending on the tabular dimensions. With Definition \ref{def:alpha_normal}, in this section, we assume that the normalization condition holds for all models in $\cM$.
\begin{assumption}
  Given the model class $\cM$ in Assumption \ref{ass:realizability_linear}, for all $h\in[H]$, the feature map $\phi_h$ satisfies $\|\phi_h(s,a)\|_2\le 1$ for all $s\in\cS$ and all $a\in\cA$. Moreover, for all $M\in\cM$, $M$ satisfies the normalization condition in Definition \ref{def:alpha_normal}, and $c_\cM=\sup_{M\in\cM}c_M<\infty$.  
\end{assumption}

Similar to the tabular setting, for any linear MDP $M=\{(\phi_h,\mu^{h+1},\theta^h)\}_{h\in[H]}$ and policy $\pi$, $M(\pi)$ is the distribution of the trajectory $(s^1,a^1,r^1,\ldots,s^H,a^H,r^H)$ when the learner follows policy $\pi$ in model $M$. Correspondingly, $\PP^{M,\pi}$ and $\EE^{M,\pi}$ denote the probability and expectation under $M(\pi)$. We also write $M(k)$ for the $k$-th layer of the model, namely $M(k)=\cbr{\phi_k,\mu^{k+1},\theta^k}$ for $k=1,2,\ldots,H$.
\paragraph{Offline Regression Oracles} For linear MDPs, instead of the Hellinger distance, we consider the $L^2$ distance between two distributions. Specifically, given any two distributions $\PP$ and $\QQ$ and some base measure $\mu$, the $L^2$ distance between them is defined as
\[
\|\PP-\QQ\|_{L^2}^2:=\int \abs{\frac{d\PP}{d\mu}-\frac{d\QQ}{d\mu}}^2d\mu.
\]
For linear MDP model estimation, we use the following offline-regression statistical guarantee in $L^2$ distance. Throughout the linear-MDP analysis, all $L^2$ norms of
transition distributions are taken with respect to the
counting measure on the state space $\cS$.
\begin{definition}\label{def:offline_oracle}
Given $n$ i.i.d. trajectories with base policies $\pi_i\sim p$, the offline regression oracle returns $\hat M\in\cM$. The observed successor states $s_i^{h+1}$, including the terminal successor $s_i^{H+1}$, are retained with the trajectory data. For every $\delta\in(0,1/2)$, with probability at least $1-\delta$, simultaneously for $h\in[H]$,
\[
\EE_{\pi\sim p}\EE^{M_*,\pi}\left[
\int_{s'}\bigl|\hat P^h(s'\mid s^h,a^h)-P_*^h(s'\mid s^h,a^h)\bigr|^2
+\bigl|\hat r^h(s^h,a^h)-r_*^h(s^h,a^h)\bigr|^2
\right]\le\cE_{\cM}^{\delta}(n).
\]
Thus the oracle controls conditional next-state and reward prediction at the true state-action distribution. This is the per-layer quantity used below, rather than an unconditioned $L^2$ distance between complete trajectory densities.
\end{definition}
For finite $\cM$, conditional least-squares estimation gives $\cE_{\cM}^\delta(n)\le40H\log(|\cM|/\delta)/n$ by Lemma~\ref{app-lemma:conditional_LSE_finite}. As in the tabular setting, we impose the following assumption on the offline regression oracle $\Alg$.
\begin{assumption}
 $\cE_\cM^\delta(n)$ is decreasing in $n$ and $\delta$. Specifically, for any $n_1<n_2$, we have $\cE_\cM^\delta(n_1)>\cE_\cM^\delta(n_2)$. For any  $\delta_1<\delta_2$, we have $\cE_\cM^{\delta_1}(n)>\cE_\cM^{\delta_2}(n)$.
\end{assumption}

\section{Double-Oracle-Efficient Reinforcement Learning: Linear MDP}\label{sec:linear_alg_design}
We present the algorithm as follows. Subsection \ref{subsec:overview_alg_linear} gives an overview and the pseudocode. Subsection \ref{subsec:truncated_log_determinant} then gives the construction details and the main ideas behind the learning algorithm.

\subsection{Algorithm Design}\label{subsec:overview_alg_linear}

As in Algorithm \ref{alg:offline_tabular_RL}, we divide the $T$ total rounds into $N$ epochs according to the schedule $0=\tau_0<\cdots<\tau_N=\frac{T}{H}$. Each epoch $m$ is evenly split into $H$ segments of $\tau_m-\tau_{m-1}$ rounds. During the $h$-th segment of epoch $m$, we fix a mixture $p^h_m\in\Pi_{\mathrm{mix}}$, independently draw $\pi_t\sim p^h_m$ before each trajectory, execute that sampled base policy throughout the trajectory, and collect trajectory data $\cbr{\pi_t}\cup\{s_t^j,a_t^j,R_t^j\}_{j=1}^{H}$. We pass this data to the offline regression oracle $\Alg$ and obtain the estimated model $\hat{M}^h_m=\cbr{(\phi_ {j},\hat{\mu}^{j+1}_{m,h},\hat{r}^ {j}_{m,h})_{j=1}^{H}}$. From this model, we extract only the $h$-th layer transition kernel $\inner{\phi_h(\cdot,\cdot)}{\hat{\mu}^{h+1}_{m,h}(\cdot)}$ and reward function $\hat{r}^{h}_{m,h}$, and aggregate these layer-wise estimates into a single approximate model for the next epoch $m+1$.

For each model $\hat{M}_m^h$, we use only its output at layer $h$. As in the tabular setting, the aggregated estimated model at the end of epoch $m$ is denoted by $\hat{M}_{m}^{out}$ and is given by
\[
\hat{M}_{m}^{out}
=
\left\{(\phi_ {j},\hat{\mu}^{j+1}_{m,j},\hat{\theta}^ {j}_{m,j})_{j=1}^{H}\right\} .
\]
We use the same value and meta-policy notation as in the tabular setting: $\hat V_m^h(s;\rho)$ and $\hat Q_m^h(s,a;\rho)$ refer to a base policy $\rho$, while $\hat V_m^1(p)=\EE_{\rho\sim p}[\hat V_m^1(\rho)]$, $p_m\in\argmax_{p\in\Pi_{\mathrm{mix}}}\hat V_m^1(p)$, and $\hatreg_m(p)=\hat V_m^1(p_m)-\hat V_m^1(p)$. Again, $p_m$ can be chosen as a Dirac distribution. At layer $h$, write $\hat P_m^h(\cdot\mid s,a)=\inner{\phi_h(s,a)}{\hat\mu_{m,h}^{h+1}(\cdot)}$ and $\hat r_m^h(s,a)=\inner{\phi_h(s,a)}{\hat\theta_{m,h}^h}$. Algorithm~\ref{alg:offline_linear_RL} gives the pseudocode.
\begin{algorithm}[h]
\begin{algorithmic}[1]
\caption{Doubly Oracle-Efficient Reinforcement Learning (\texttt{DOERL}): Linear MDP}\label{alg:offline_linear_RL}
\REQUIRE epoch schedule $0=\tau_0<\tau_1<\cdots<\tau_N=T/H$, confidence parameter $0<\delta<1/2$, model class $\cM$, offline regression oracle $\Alg$.
\STATE Initialize: $\hat{M}_0^{out}$ to be any normalized admissible linear MDP model with state space $\cS$ and action space $\cA$.
\FOR{epoch $m=1,2,\cdots,N$}
\STATE Compute $p_{m-1}\in\argmax_{p\in\Pi_{\mathrm{mix}}}\hat V_{m-1}^1(p)$ with one unregularized planning call (a Dirac optimizer suffices), to evaluate $\hatreg_{m-1}$ in the trusted-set thresholds.
\STATE Set $\delta_m,E_m,\beta_m,\zeta_m,\eta_m$ as in Theorem~\ref{thm:main_reg_linear}.
 \FOR{segment $h=1,\cdots,H$}
 \STATE {Compute
 \[
p^h_m\in\argmax_{p\in\Pi_{\mathrm{mix}}}\cbr{\hat{V}_{m-1}^1(p)+\frac{1}{\eta_m}\sum_{j=1}^{h}\log\rbr{\det\rbr{\tilde{K}^j_m(p)+\beta_mI_d}}},
 \]
 where $\tilde{K}^j_m(p)$ is defined in Subsection \ref{subsec:truncated_log_determinant}.}
 \FOR{\parbox[t]{.75\linewidth}{$t=\tau_{m-1}H+(\tau_m-\tau_{m-1})(h-1)+1,\cdots,\tau_{m-1}H+(\tau_m-\tau_{m-1})h$}}
 
 \STATE Independently draw $\pi_t\sim p^h_m$ and execute $\pi_t$ for the entire trajectory.
 \STATE Observe the trajectory $\cbr{\pi_t,s_t^1,a_t^1,R_t^1,\cdots,s_t^H,a_t^H,R_t^H,s_t^{H+1}}$.
 
 \ENDFOR
\STATE Run the offline regression oracle $\Alg$ at confidence $\delta_m$, with the trajectory data points $\{\pi_t,\{s_t^j,a_t^j,R_t^j\}_{j\in[H]}\big|\tau_{m-1}H+(\tau_m-\tau_{m-1})(h-1)<t\le\tau_{m-1}H+(\tau_m-\tau_{m-1})h\}$.
\STATE Obtain the estimated model $\hat{M}^h_m=\cbr{(\phi_ {j},\hat{\mu}^{j+1}_{m,h},\hat{r}^ {j}_{m,h})_{j=1}^{H}} $ 
 \ENDFOR
 \STATE Formulate the model $\hat{M}_m^{out}=\cbr{(\phi_ {h},\hat{\mu}^{h+1}_{m,h},\hat{r}^ {h}_{m,h})_{h=1}^{H}}$.
\ENDFOR
\end{algorithmic}
\end{algorithm}
\subsection{Per-Epoch Detailed Construction: Linear MDP}\label{subsec:truncated_log_determinant}
At the core of our technical contribution in this section is the development of a novel trusted occupancy measure, coupled with a corresponding truncated log-determinant regularization. We begin by formulating a trusted probability transition set and its associated occupancy measure, both of which diverge from their tabular counterparts established in Section \ref{sec:tabular_alg_design}.

For the linear-MDP construction, write $E_m\ge\cE_{\cM}^{\delta_m}(\tau_m-\tau_{m-1})$ with the confidence allocation $\delta_m$ specified in Theorem~\ref{thm:main_reg_linear}, and set
\[
 v_m=e^2E_m+4c_{\cM}^4\beta_m,
 \ 
 B_m(p)=Hd+\eta_m\hatreg_{m-1}(p).
\]
Here and below we may enlarge $c_{\cM}$ so that $c_{\cM}\ge1$. The parameter $\zeta_m$ is chosen so that
\[
 \zeta_m^{-1}=(H+1)\sqrt{v_m}.
\]
For $p=\sum_iw_i\delta_{\rho_i}$, the affine pseudo-regret gives $B_m(p)=\sum_iw_iB_m(\rho_i)$, where $B_m(\rho_i)$ means $B_m(\delta_{\rho_i})$. The trusted sets below are indexed by the base policy $\rho$; their entire mapping $\rho\mapsto\tilde{\cT}_m^j(\rho)$ is fixed when optimizing over meta-policies in a segment.
\begin{definition}\label{def:trusted_occupancy_measure_linear}
For each base Markov policy $\rho\in\Pi_{RNS}$, define
\[
 \tilde d_m^1(s;\rho)=\II(s=s^1),
 \ 
 \tilde d_m^1(s,a;\rho)=\II(s=s^1)\rho^1(a\mid s).
\]
For $h=1,\ldots,H-1$, let
\[
 \tilde p_m^{h+1}(s';\rho)
 =\int_{s,a}\tilde d_m^h(s,a;\rho)\hat P_m^h(s'\mid s,a),
 \ 
 \tilde{\cT}_m^{h+1}(\rho)
 =\left\{s':\tilde p_m^{h+1}(s';\rho)
       \ge\frac{\sqrt{B_m(\rho)}}{\zeta_m}\right\}.
\]
The trusted occupancy measures are
\[
 \tilde d_m^{h+1}(s';\rho)
 =\tilde p_m^{h+1}(s';\rho)\II(s'\in\tilde{\cT}_m^{h+1}(\rho)),
 \ 
 \tilde d_m^{h+1}(s',a';\rho)
 =\tilde d_m^{h+1}(s';\rho)\rho^{h+1}(a'\mid s').
\]
Here $\tilde p_m^{h+1}(s';\rho)$ is the estimated arrival mass before truncation; it is a scalar, distinct from the meta-policy $p$. For $p=\sum_iw_i\delta_{\rho_i}\in\Pi_{\mathrm{mix}}$, average the base-policy occupancies:
\[
 \tilde d_m^h(s,a;p)
 =\sum_iw_i\tilde d_m^h(s,a;\rho_i),
 \ 
 \tilde d_m^h(s;p)
 =\sum_iw_i\tilde d_m^h(s;\rho_i).
\]
\end{definition}
For the true model, the true observable occupancy of a base policy is defined using the same policy-indexed sets:
\[
 d_m^1(s;\rho)=\II(s=s^1),\ d_m^1(s,a;\rho)=\II(s=s^1)\rho^1(a\mid s),
\]
\[
 d_m^{h+1}(s';\rho)
 =\int_{s,a}d_m^h(s,a;\rho)P_*^h(s'\mid s,a)
     \II(s'\in\tilde{\cT}_m^{h+1}(\rho)),
 \ 
 d_m^{h+1}(s',a';\rho)=d_m^{h+1}(s';\rho)\rho^{h+1}(a'\mid s').
\]
For $p=\sum_iw_i\delta_{\rho_i}$, set $d_m^h(\cdot;p)=\sum_iw_i d_m^h(\cdot;\rho_i)$. Each component uses its own sets $\tilde{\cT}_m^j(\rho_i)$ under both kernels; no trusted set is constructed from an averaged arrival mass. The covariance matrices $\tilde K_m^h(p)$ and $K_m^h(p)$ below are defined from these averaged occupancies. In particular, they are affine in $p$; the base-policy truncations are not recomputed after mixing. All kernels needed for the layer-$h$ objective have already been estimated in segments $1,\ldots,h-1$ of the current epoch.

Recall that after completing all $H$ segments in epoch $m$, we obtain the model $\hat{M}_m^{out}=\left\{(\phi_ {j},\hat{\mu}^{j+1}_{m,j},\hat{r}^ {j}_{m,j})_{j=1}^{H}\right\} $, which may be viewed as a summary of epoch $m$. For this model $\hat{M}_{m}^{out}$, we define the \textbf{estimated occupancy measure} as
\[
\hat{d}^h_m(s;p) :=\EE^{\hat{M}^{out}_{m},p }[\II(s^h=s)],\ \hat{d}^h_m(s,a;p) :=\EE^{\hat{M}^{out}_{m},p }[\II(s^h=s,a^h=a)].
\]
Given the trusted occupancy of a base policy $\rho$, define its feature covariance matrix by
\[
\tilde K_m^h(\rho):=\int_{s,a}\phi_h(s,a)\phi_h(s,a)^T\tilde d_m^h(s,a;\rho).
\]
For $p=\sum_iw_i\delta_{\rho_i}\in\Pi_{\mathrm{mix}}$, its averaged occupancy therefore gives
\[
\tilde K_m^h(p):=\sum_iw_i\tilde K_m^h(\rho_i)
=\int_{s,a}\phi_h(s,a)\phi_h(s,a)^T\tilde d_m^h(s,a;p).
\]
Intuitively, $\tilde{d}^h_m$ is a truncation of the estimated occupancy measure $\hat{d}^h_m$, and $\tilde{K}^h_m(p) $ integrates $\phi_h\phi_h^T$ over the sub-probability measure $\tilde{d}^h_m$.

With the truncated feature matrix $\tilde K_m^h$, for $p\in\Pi_{\mathrm{mix}}$ we define the truncated log-determinant regularization term as
\[
\Lambda^h_m(p) :=\sum_{j=1}^{h}\log\rbr{\det(\tilde{K}^j_m(p) +\beta_mI_d)}.
\]
The $\beta_mI_d$ term ensures that the matrix is positive definite, so the log-determinant is well defined. Beyond the weighting by the trusted occupancy measure, our log-determinant regularization has another important feature. By definition, $\Lambda^h_m$ is the sum of $h$ log-determinant terms, $\log(\det(\tilde{K}^j_m+\beta_mI_d))$ for $j=1,\dots,h$. Thus, instead of exploring the $h$-th layer in isolation, the algorithm encourages joint exploration across the first $h$ layers. This is another important departure from Algorithm \ref{alg:offline_tabular_RL} for tabular MDPs.

Therefore, in segment $h$ of epoch $m$, the planning oracle selects the meta-policy
\[
p^h_m\in\argmax_{p\in\Pi_{\mathrm{mix}}}\cbr{\hat{V}^1_{m-1}(p)+\frac{1}{\eta_m}\Lambda_m^h(p)}.
\]
The objective function of the planning oracle is $J^h_m(p)=\hat{V}^1_{m-1}(p)+\frac{1}{\eta_m}\Lambda_m^h(p)$. This policy has a natural interpretation: it maximizes a weighted objective that balances exploration and exploitation. The first term, $\hat{V}_{m-1}^1(p)$, captures exploitation because maximizing this value function corresponds to maximizing the total expected reward along the trajectory. The log-determinant term of the occupancy-weighted feature matrix encourages exploration \citep{foster2020adapting}; maximizing it is closely related to maximizing entropy \citep{zhang2021exploration}, which promotes exploration of the environment \citep{dai2023refined}. Finally, by tuning $\eta_m$, the policy balances exploration and exploitation and controls the overall regret.

The log-determinant was previously analyzed by \cite{foster2020adapting} in the context of contextual bandits. They introduced it as a generalization of log-barrier regularization, effectively extending the inverse gap weighting policy to accommodate infinite action spaces. Drawing on this relationship, it becomes clear that our log-determinant regularized policy naturally balances exploration and exploitation.

Similarly, for a base policy $\rho$, the true observable and full estimated feature covariance matrices are
\begin{align*}
K_m^j(\rho)&:=\int_{s,a}\phi_j(s,a)\phi_j(s,a)^T d_m^j(s,a;\rho),\\
\hat K_m^j(\rho)&:=\int_{s,a}\phi_j(s,a)\phi_j(s,a)^T\hat d_m^j(s,a;\rho).
\end{align*}
For $p=\sum_iw_i\delta_{\rho_i}$, set $K_m^j(p):=\sum_iw_iK_m^j(\rho_i)$ and $\hat K_m^j(p):=\sum_iw_i\hat K_m^j(\rho_i)$. By the occupancy averages, these also equal the corresponding integrals against $d_m^j(\cdot;p)$ and $\hat d_m^j(\cdot;p)$.
\section{Theoretical Guarantees: Linear MDPs}\label{sec:theory_linear}
In this section, we present the theoretical guarantees. Subsection \ref{subsec:theory_per_epoch_linear} gives the per-epoch guarantees and key lemmas. Subsection \ref{subsec:theory_regret_analysis_linear} then states and proves the final regret bound.
\subsection{Per-Epoch Theoretical Analysis}\label{subsec:theory_per_epoch_linear}
This subsection provides a sequence of lemmas for a fixed epoch $m$. These lemmas play the same role as the per-epoch lemmas in Subsection \ref{subsec:theory_perepoch_tabular}. Their proofs are provided in the Appendix.

We first show that the executed meta-policy $p^h_m $ satisfies a trace-control property uniformly over all other policies in $\Pi_{\mathrm{mix}}$. This property is central to the analysis and parallels Lemma \ref{lemma:hatpi_property_tabular} in Section \ref{sec:theory_tabular}.
\begin{lemma}\label{lemma:hatpi_property_linear}
    For any $m,h$ and $p\in\Pi_{\mathrm{mix}}$, we have
    \begin{align*}
    \sum_{j=1}^{h}\tr((\tilde{K}^j_m(p^h_m ) +\beta_mI_d)^{-1}\tilde{K}^j_m(p) )\le& hd+\eta_m(\hat{\mathrm{reg}}_{m-1}(p) -\hat{\mathrm{reg}}_{m-1}(p^h_m ) )\\
    \le&hd+\eta_m\hat{\mathrm{reg}}_{m-1}(p).
    \end{align*}
\end{lemma}
Intuitively, in a linear MDP, an effective exploration policy $p$ should seek to ``maximize'' the feature covariance matrix $\hat{K}_m^h(p)$ of the estimated model in the Loewner order, as discussed in \cite{foster2020adapting}. For any two positive semi-definite matrices $A$ and $B$, the condition $\operatorname{tr}(A^{-1}B) \leq c$ suggests, at a heuristic level, that $B$ is controlled by $A$ up to a factor of order $c$ in the Loewner order. Therefore, if we can show that $\tilde{d}_m^h$ is sufficiently close to $\hat{d}_m^h$ in a suitable aggregate sense, then the trace-control guarantee in Lemma~\ref{lemma:hatpi_property_linear} implies that $p_m^h$ provides effective exploration over the first $h$ steps of the unknown underlying model $M_*$.

Our next goal is thus to show that the trusted occupancy measure closely approximates the estimated occupancy measure $\hat{d}_m^h$. In tabular MDPs, Lemma~\ref{lemma:tilded_boundby_hatd_tabular} provides a pointwise comparison between $\hat{d}_m^h$ and $\tilde{d}_m^h$ for every state-action pair. Such a fine-grained guarantee is no longer available in linear MDPs, where the state space $\cS$ may be infinite and the action space $\cA$ may be arbitrary. Consequently, for an arbitrary meta-policy $p$, we instead control the discrepancy between the corresponding integrated occupancy measures. Specifically, we establish the following lemma.
\begin{lemma}\label{lemma:trusted_occupancy_boundby_estimated_occupancy}
For every epoch $m$, layer $h$, and meta-policy $p\in\Pi_{\mathrm{mix}}$,
\[
 \int_{s,a}\bigl(\hat d_m^h(s,a;p)-\tilde d_m^h(s,a;p)\bigr)
 \le\frac{hc_{\cM}}{\sqrt{\zeta_m}}\,B_m(p)^{1/4},
 \  B_m(p)=Hd+\eta_m\hatreg_{m-1}(p).
\]
The same bound holds for a base Markov policy, identified with its Dirac mixture.
\end{lemma}
The factor $B_m(p)^{1/4}$ follows from the policy-indexed truncation in Definition~\ref{def:trusted_occupancy_measure_linear}. The next lemma compares trusted occupancies under the estimated and true kernels without changing the policy inside the induction.
\begin{lemma}\label{lemma:tilde_dP<dP*_linear}
Fix $m$ and suppose that, for every $k\in[H]$, the segment-$k$ estimate satisfies
\[
 \EE^{M_*,p_m^k}\left[
 \int_{s'}\bigl|\hat P_m^k(s'\mid s^k,a^k)-P_*^k(s'\mid s^k,a^k)\bigr|^2
 +\bigl|\hat r_m^k(s^k,a^k)-r_*^k(s^k,a^k)\bigr|^2
 \right]\le E_m.
\]
Let $v_m=e^2E_m+4c_{\cM}^4\beta_m$ and choose
\[
 \zeta_m^{-1}=(H+1)\sqrt{v_m}.
\]
Then, simultaneously for every base Markov policy $\rho$ and layer $j$,
\[
 \tilde d_m^j(s,a;\rho)
 \le\left(1+\frac1H\right)^{j-1}d_m^j(s,a;\rho).
\]
Consequently, for every meta-policy $p\in\Pi_{\mathrm{mix}}$ and $u\in\RR^d$,
\[
 \int_{s,a}|\inner{u}{\phi_j(s,a)}|\tilde d_m^j(s,a;p)
 \le\left(1+\frac1H\right)^{2(j-1)}
 \int_{s,a}|\inner{u}{\phi_j(s,a)}|d_m^j(s,a;p),
\]
\[
 \tilde K_m^j(p)
 \preceq\left(1+\frac1H\right)^{2(j-1)}K_m^j(p).
\]
\end{lemma}
The comparison is uniform over base policies and therefore extends to mixtures by averaging. Together with the truncation bound, it yields the following value-estimation guarantee. No dimension-free inequality between $L^1$ and counting-measure $L^2$ norms is used.
\begin{lemma}\label{lemma:perepoch_valuefunc_estimationerror_linear}
Suppose that the oracle event in Lemma~\ref{lemma:tilde_dP<dP*_linear} holds and $0<\beta_m\le1$. Then, for every meta-policy $p\in\Pi_{\mathrm{mix}}$,
\begin{align*}
 \left|\hat V_m^1(p)-V_*^1(p)\right|
 \le{}&2H^2c_{\cM}\sqrt{H+1}\,v_m^{1/4}B_m(p)^{1/4}\\
 &+H\left[(2c_{\cM})^{2/3}2^{1/6}v_m^{1/3}
                 +\sqrt{v_m}\right]\sqrt{B_m(p)}.
\end{align*}
In particular, with $C_0=100(1+c_{\cM})^4H^3$,
\[
 \left|\hat V_m^1(p)-V_*^1(p)\right|
 \le\min\left\{1,\ C_0v_m^{1/4}
             \bigl(Hd+\eta_m\hatreg_{m-1}(p)\bigr)\right\}.
\]
\end{lemma}

Since the bound in Lemma \ref{lemma:perepoch_valuefunc_estimationerror_linear} depends on the pseudo-regret $\hat{\mathrm{reg}}_{m-1}(p)$, Lemma \ref{lemma:pseudo_regret_linear} shows that this pseudo-regret is well controlled for linear MDPs.
\begin{lemma}\label{lemma:pseudo_regret_linear}
    For any $m,h$, we have
\[
\hat{\mathrm{reg}}_{m-1}(p^h_{m})\le \frac{hd}{\eta_m}\log(\frac{1}{d}+\beta_m)-\frac{hd\log\beta_m}{\eta_m}.
\]
\end{lemma}

\subsection{Regret Analysis}\label{subsec:theory_regret_analysis_linear}
In this section, we give the main regret guarantee for linear MDPs. As in Section \ref{sec:theory_tabular}, we first state the main result and then specialize it to finite model classes.
\begin{theorem}\label{thm:main_reg_linear}
Assume the per-layer offline-regression guarantee stated in Definition~\ref{def:offline_oracle}. Fix $\delta\in(0,1)$ and an epoch schedule satisfying $n_m:=\tau_m-\tau_{m-1}\ge\tau_{m-1}$. If $N$ is known, use $\delta_m=\delta/[4H(N+1)^2]$; otherwise use $\delta_m=\delta/[4H(m+1)^2]$. Let $\varepsilon_m>0$ be a deterministic nonincreasing upper bound on $\cE_\cM^{\delta_m}(n_m)$, and put
\[
 E_m=\varepsilon_m,\;
 \beta_m=\min\{1,E_m\}^2,\;
 v_m=e^2E_m+4c_\cM^4\beta_m.
\]
Define
\[
 C_0=100(1+c_\cM)^4H^3,\;
 a_0=1,\;
 a_m=\min\{1,C_0Hd\,v_m^{1/4}\},
\]
and choose
\[
 \eta_m=\frac{1}{100a_{m-1}},\;
 \zeta_m=\frac{1}{(H+1)\sqrt{v_m}},\;
 L_m=1+\log\!\left(1+\frac{1}{d\beta_m}\right).
\]
Then, on the simultaneous offline-oracle event, Algorithm~\ref{alg:offline_linear_RL} with episode-level mixtures and the policy-specific trusted sets in Definition~\ref{def:trusted_occupancy_measure_linear} satisfies
\[
 \mathrm{Reg}(T)
 \le H\sum_{m=1}^N n_m\min\{1,210Hd\,L_m a_{m-1}\}.
\]
The confidence allocation above makes this event have probability at least $1-\delta$. The numbers of offline-estimation and policy-optimization oracle calls are at most $NH$ and $(H+1)N$, respectively; both are $O(HN)$.
\end{theorem}

For a finite model class, Lemma~\ref{app-lemma:conditional_LSE_finite} gives the per-layer conditional least-squares error bound $\cE_\cM^\delta(n)\le40H\log(|\cM|/\delta)/n$. When $N$ is known, a constant confidence logarithm gives a nonincreasing sequence $\varepsilon_m$; for doubling epochs use $80H\log(4H|\cM|(m+1)^2/\delta)/n_m$, enlarging its first term if necessary. The following corollaries yield regret $\widetilde O(T^{3/4})$ with the same oracle-call orders as before. Write $\ell_N=\max\{1,80H\log(4H|\cM|(N+1)^2/\delta)\}$.

\begin{corollary}\label{cor:reg_known_T_linear}
Suppose $T$ is known and put $B=T/H$. For $B\ge16$, use the schedule $\tau_m=2B^{1-2^{-m}}$ until its last value not exceeding $B/2$, and append the endpoint $B$; integer rounding changes only absolute constants. Thus $N=O(\log\log T)$. With the parameters of Theorem~\ref{thm:main_reg_linear}, with probability at least $1-\delta$,
\[
 \mathrm{Reg}(T)
 \lesssim C_0(1+c_\cM)H^{1/4}(Hd)^2
 \ell_N^{1/4}\log(2+T)\,N\,T^{3/4}.
\]
Consequently the regret is $\widetilde O(\operatorname{poly}(H,d,c_\cM)T^{3/4})$, and each oracle is called $O(H\log\log T)$ times. Bounded $B$ is covered by the trivial regret bound.
\end{corollary}

\begin{corollary}\label{cor:reg_unknown_T_linear}
Suppose $T$ is unknown and use $\tau_m=2^m$, truncating the execution of the final epoch at the horizon. With the parameters of Theorem~\ref{thm:main_reg_linear}, with probability at least $1-\delta$,
\[
 \mathrm{Reg}(T)
 \lesssim C_0(1+c_\cM)H^{1/4}(Hd)^2
 \ell_N^{1/4}\log(2+T)\,T^{3/4}.
\]
Thus the regret is $\widetilde O(\operatorname{poly}(H,d,c_\cM)T^{3/4})$, and each oracle is called $O(H\log T)$ times.
\end{corollary}

These two corollaries show that our algorithm achieves sublinear regret while requiring only low-frequency calls to the offline regression oracle. To the best of our knowledge, this is the first algorithm with such a guarantee for MDPs with infinite state and action spaces.

\section{Discussion}
In this paper, we introduced \texttt{DOERL}, an algorithm that attains the optimal regret bound with only $O(H\log\log T)$ calls to the offline estimation and planning oracles when $T$ is known. This is the first offline-oracle-efficient approach whose planning complexity is independent of the sizes of the state and action spaces. We further extended \texttt{DOERL} to linear MDPs with infinite state spaces and arbitrary action spaces, yielding the first offline-oracle-efficient regret-minimization algorithm for MDPs beyond the tabular model.

Our theoretical framework suggests several directions for future research. First, the regret bound established here for linear MDPs is suboptimal. A natural open problem is to develop an offline-oracle-efficient algorithm for linear MDPs with infinite state and action spaces that attains the optimal regret rate. Second, our results raise broader questions about the fundamental limits of offline-oracle-efficient learning. Offline oracle-efficient regret-minimization algorithms are now known for contextual bandits , tabular MDPs, and linear MDPs, but it remains unclear whether general sequential decision-making problems can be solved efficiently using only low-frequency calls to offline oracles. Characterizing the capability gap between offline and online oracles is an important open problem. Finally, implementing these oracle-efficient algorithms in practical settings and evaluating their empirical performance is another important direction. We leave these questions for future work.

\bibliographystyle{tmlr}
\bibliography{refs}

@misc{agarwal2019reinforcement,
  author = {Agarwal, A. and Jiang, N. and Kakade, S. M. and Sun, W.},
  title = {Reinforcement Learning: Theory and Algorithms},
  year = {2019}
}

@inproceedings{auer2008near,
  author = {Auer, P. and Jaksch, T. and Ortner, R.},
  title = {Near-Optimal Regret Bounds for Reinforcement Learning},
  booktitle = {Advances in Neural Information Processing Systems},
  volume = {21}, year = {2008}
}

@inproceedings{ayoub2020model,
  author = {Ayoub, A. and Jia, Z. and Szepesvari, C. and Wang, M. and Yang, L.},
  title = {Model-Based Reinforcement Learning with Value-Targeted Regression},
  booktitle = {International Conference on Machine Learning},
  publisher = {PMLR}, pages = {463--474}, year = {2020}
}

@misc{azar2017minimaxregretboundsreinforcement,
  author = {Azar, M. G. and Osband, I. and Munos, R.},
  title = {Minimax Regret Bounds for Reinforcement Learning},
  year = {2017}, eprint = {1703.05449}, archivePrefix = {arXiv},
  url = {https://arxiv.org/abs/1703.05449}
}

@inproceedings{bai2019provably,
  author = {Bai, Y. and Xie, T. and Jiang, N. and Wang, Y.-X.},
  title = {Provably Efficient {Q}-Learning with Low Switching Cost},
  booktitle = {Advances in Neural Information Processing Systems},
  volume = {32}, year = {2019}
}

@article{balseiro2019dynamic,
  author = {Balseiro, S. R. and Brown, D. B. and Chen, C.},
  title = {Dynamic Pricing of Relocating Resources in Large Networks},
  journal = {ACM SIGMETRICS Performance Evaluation Review},
  volume = {47}, pages = {29--30}, year = {2019}
}

@article{bellman1966dynamic,
  author = {Bellman, R.}, title = {Dynamic Programming},
  journal = {Science}, volume = {153}, pages = {34--37}, year = {1966}
}

@misc{cesani2026log,
  author = {Cesani, L. and Papini, M. and Restelli, M.},
  title = {How Log-Barrier Helps Exploration in Policy Optimization},
  year = {2026}, eprint = {2603.15001}, archivePrefix = {arXiv},
  url = {https://arxiv.org/abs/2603.15001}
}

@article{chen2025unified,
  author = {Chen, F. and Mei, S. and Bai, Y.},
  title = {Unified Algorithms for {RL} with Decision-Estimation Coefficients: {PAC}, Reward-Free, Preference-Based Learning and Beyond},
  journal = {The Annals of Statistics}, volume = {53},
  pages = {426--456}, year = {2025}
}

@misc{chen2026information,
  author = {Chen, K.},
  title = {Information-Directed Offline-to-Online Reinforcement Learning},
  year = {2026}, eprint = {2605.29405}, archivePrefix = {arXiv},
  url = {https://arxiv.org/abs/2605.29405}
}

@inproceedings{dai2023refined,
  author = {Dai, Y. and Luo, H. and Wei, C.-Y. and Zimmert, J.},
  title = {Refined Regret for Adversarial {MDPs} with Linear Function Approximation},
  booktitle = {International Conference on Machine Learning},
  publisher = {PMLR}, pages = {6726--6759}, year = {2023}
}

@misc{dai2023safe,
  author = {Dai, J. and Pan, X. and Sun, R. and Ji, J. and Xu, X. and Liu, M. and Wang, Y. and Yang, Y.},
  title = {Safe {RLHF}: Safe Reinforcement Learning from Human Feedback},
  year = {2023}, eprint = {2310.12773}, archivePrefix = {arXiv},
  url = {https://arxiv.org/abs/2310.12773}
}

@inproceedings{deng2024sample,
  author = {Deng, J. and Cheng, Y. and Zou, S. and Liang, Y.},
  title = {Sample Complexity Characterization for Linear Contextual {MDPs}},
  booktitle = {International Conference on Artificial Intelligence and Statistics},
  publisher = {PMLR}, pages = {1693--1701}, year = {2024}
}

@inproceedings{foster2016learning,
  author = {Foster, D. J. and Li, Z. and Lykouris, T. and Sridharan, K. and Tardos, E.},
  title = {Learning in Games: Robustness of Fast Convergence},
  booktitle = {Advances in Neural Information Processing Systems},
  volume = {29}, year = {2016}
}

@article{foster2020adapting,
  title={Adapting to misspecification in contextual bandits},
  author={Foster, Dylan J and Gentile, Claudio and Mohri, Mehryar and Zimmert, Julian},
  journal={Advances in Neural Information Processing Systems},
  volume={33},
  pages={11478--11489},
  year={2020}
}

@inproceedings{foster2020beyond,
  title={Beyond ucb: Optimal and efficient contextual bandits with regression oracles},
  author={Foster, Dylan and Rakhlin, Alexander},
  booktitle={International Conference on Machine Learning},
  pages={3199--3210},
  year={2020},
  organization={PMLR}
}

@article{foster2021statistical,
  title={The statistical complexity of interactive decision making},
  author={Foster, Dylan J and Kakade, Sham M and Qian, Jian and Rakhlin, Alexander},
  journal={arXiv preprint arXiv:2112.13487},
  year={2021}
}

@article{foster2024online,
  title={Online estimation via offline estimation: An information-theoretic framework},
  author={Foster, Dylan J and Han, Yanjun and Qian, Jian and Rakhlin, Alexander},
  journal={arXiv preprint arXiv:2404.10122},
  year={2024}
}

@misc{gao2021provably,
  author = {Gao, M. and Xie, T. and Du, S. S. and Yang, L. F.},
  title = {A Provably Efficient Algorithm for Linear Markov Decision Process with Low Switching Cost},
  year = {2021}, eprint = {2101.00494}, archivePrefix = {arXiv},
  url = {https://arxiv.org/abs/2101.00494}
}

@inproceedings{hazan2016computational,
  author = {Hazan, E. and Koren, T.},
  title = {The Computational Power of Optimization in Online Learning},
  booktitle = {Proceedings of the Forty-Eighth Annual ACM Symposium on Theory of Computing},
  pages = {128--141}, year = {2016}
}

@article{hazan2016introduction,
  author = {Hazan, E.}, title = {Introduction to Online Convex Optimization},
  journal = {Foundations and Trends in Optimization},
  volume = {2}, pages = {157--325}, year = {2016}
}

@inproceedings{hazan2019provably,
  author = {Hazan, E. and Kakade, S. and Singh, K. and Van Soest, A.},
  title = {Provably Efficient Maximum Entropy Exploration},
  booktitle = {International Conference on Machine Learning},
  publisher = {PMLR}, pages = {2681--2691}, year = {2019}
}

@misc{hu2025constrainedonlinedecisionmakingunified,
      title={Eluder-Dimension-Free Regret for Stage-Wise Constrained Contextual Bandits via Offline Regression}, 
      author={Haichen Hu and David Simchi-Levi and Navid Azizan},
      year={2025},
      eprint={2505.07101},
      archivePrefix={arXiv},
      primaryClass={stat.ML},
      url={https://arxiv.org/abs/2505.07101}, 
}

@inproceedings{
hu2025contextual,
title={Contextual Online Decision Making with Infinite-Dimensional Functional Regression},
author={Haichen Hu and Rui Ai and Stephen Bates and David Simchi-Levi},
booktitle={Forty-second International Conference on Machine Learning},
year={2025},
url={https://openreview.net/forum?id=hFnM9AqT5A}
}

@misc{hu2026interleavedresamplingrefittingdata,
  author = {Hu, H. and Simchi-Levi, D.},
  title = {Interleaved Resampling and Refitting: Data and Compute-Efficient Evaluation of Black-Box Predictors},
  year = {2026}, eprint = {2603.14218}, archivePrefix = {arXiv},
  url = {https://arxiv.org/abs/2603.14218}
}

@inproceedings{jin2018q,
  author = {Jin, C. and Allen-Zhu, Z. and Bubeck, S. and Jordan, M. I.},
  title = {Is {Q}-Learning Provably Efficient?},
  booktitle = {Advances in Neural Information Processing Systems},
  volume = {31}, year = {2018}
}

@inproceedings{jin2020provably,
  title={Provably efficient reinforcement learning with linear function approximation},
  author={Jin, Chi and Yang, Zhuoran and Wang, Zhaoran and Jordan, Michael I},
  booktitle={Conference on learning theory},
  pages={2137--2143},
  year={2020},
  organization={PMLR}
}

@inproceedings{jin2020reward,
  author = {Jin, C. and Krishnamurthy, A. and Simchowitz, M. and Yu, T.},
  title = {Reward-Free Exploration for Reinforcement Learning},
  booktitle = {International Conference on Machine Learning},
  publisher = {PMLR}, pages = {4870--4879}, year = {2020}
}

@inproceedings{jin2023no,
  author = {Jin, T. and Liu, J. and Rouyer, C. and Chang, W. and Wei, C.-Y. and Luo, H.},
  title = {No-Regret Online Reinforcement Learning with Adversarial Losses and Transitions},
  booktitle = {Advances in Neural Information Processing Systems},
  volume = {36}, pages = {38520--38585}, year = {2023}
}

@article{kastius2022dynamic,
  author = {Kastius, A. and Schlosser, R.},
  title = {Dynamic Pricing under Competition Using Reinforcement Learning},
  journal = {Journal of Revenue and Pricing Management},
  volume = {21}, pages = {50--63}, year = {2022}
}

@inproceedings{kim2022improved,
  author = {Kim, Y. and Yang, I. and Jun, K.-S.},
  title = {Improved Regret Analysis for Variance-Adaptive Linear Bandits and Horizon-Free Linear Mixture {MDPs}},
  booktitle = {Advances in Neural Information Processing Systems},
  volume = {35}, pages = {1060--1072}, year = {2022}
}

@article{kiran2021deep,
  author = {Kiran, B. R. and Sobh, I. and Talpaert, V. and Mannion, P. and Al Sallab, A. A. and Yogamani, S. and P{\'e}rez, P.},
  title = {Deep Reinforcement Learning for Autonomous Driving: A Survey},
  journal = {IEEE Transactions on Intelligent Transportation Systems},
  volume = {23}, pages = {4909--4926}, year = {2021}
}

@article{kober2013reinforcement,
  author = {Kober, J. and Bagnell, J. A. and Peters, J.},
  title = {Reinforcement Learning in Robotics: A Survey},
  journal = {The International Journal of Robotics Research},
  volume = {32}, pages = {1238--1274}, year = {2013}
}

@misc{lee2025analysis,
  author = {Lee, D. and Yang, H. and Park, B. G.},
  title = {Analysis of Approximate Linear Programming Solution to Markov Decision Problem with Log Barrier Function},
  year = {2025}, eprint = {2509.19800}, archivePrefix = {arXiv},
  url = {https://arxiv.org/abs/2509.19800}
}

@inproceedings{levy2023optimism,
  author = {Levy, O. and Mansour, Y.},
  title = {Optimism in Face of a Context: Regret Guarantees for Stochastic Contextual {MDP}},
  booktitle = {Proceedings of the AAAI Conference on Artificial Intelligence},
  volume = {37}, pages = {8510--8517}, year = {2023}
}

@inproceedings{levy2024eluder,
  author = {Levy, O. and Cassel, A. and Cohen, A. and Mansour, Y.},
  title = {Eluder-Based Regret for Stochastic Contextual {MDPs}},
  booktitle = {International Conference on Machine Learning},
  publisher = {PMLR}, pages = {27326--27350}, year = {2024}
}

@misc{levy2026nearoptimalregretpolicyoptimization,
  author = {Levy, O. and Rosenberg, A. and Cohen, A. and Mansour, Y.},
  title = {Near-Optimal Regret for Policy Optimization in Contextual {MDPs} with General Offline Function Approximation},
  year = {2026}, eprint = {2602.13706}, archivePrefix = {arXiv},
  url = {https://arxiv.org/abs/2602.13706}
}

@inproceedings{liu2020ipo,
  author = {Liu, Y. and Ding, J. and Liu, X.},
  title = {{IPO}: Interior-Point Policy Optimization under Constraints},
  booktitle = {Proceedings of the AAAI Conference on Artificial Intelligence},
  volume = {34}, pages = {4940--4947}, year = {2020}
}

@inproceedings{liu2023bypassing,
  author = {Liu, H. and Wei, C.-Y. and Zimmert, J.},
  title = {Bypassing the Simulator: Near-Optimal Adversarial Linear Contextual Bandits},
  booktitle = {Advances in Neural Information Processing Systems},
  volume = {36}, pages = {52086--52131}, year = {2023}
}

@inproceedings{mhammedi2023efficient,
  author = {Mhammedi, Z. and Block, A. and Foster, D. J. and Rakhlin, A.},
  title = {Efficient Model-Free Exploration in Low-Rank {MDPs}},
  booktitle = {Advances in Neural Information Processing Systems},
  volume = {36}, pages = {66782--66817}, year = {2023}
}

@misc{misra2023bellman,
  author = {Misra, R. and Wisniewski, R. and Kalles{\o}e, C. S.},
  title = {On {Bellman}'s Principle of Optimality and Reinforcement Learning for Safety-Constrained Markov Decision Process},
  year = {2023}, eprint = {2302.13152}, archivePrefix = {arXiv},
  url = {https://arxiv.org/abs/2302.13152}
}

@article{modi2024model,
  author = {Modi, A. and Chen, J. and Krishnamurthy, A. and Jiang, N. and Agarwal, A.},
  title = {Model-Free Representation Learning and Exploration in Low-Rank {MDPs}},
  journal = {Journal of Machine Learning Research},
  volume = {25}, pages = {1--76}, year = {2024}
}

@inproceedings{ni2025safe,
  author = {Ni, T. and Kamgarpour, M.},
  title = {A Safe Exploration Approach to Constrained Markov Decision Processes},
  booktitle = {International Conference on Artificial Intelligence and Statistics},
  publisher = {PMLR}, pages = {3592--3600}, year = {2025}
}

@article{ouyang2022training,
  title={Training language models to follow instructions with human feedback},
  author={Ouyang, Long and Wu, Jeffrey and Jiang, Xu and Almeida, Diogo and Wainwright, Carroll and Mishkin, Pamela and Zhang, Chong and Agarwal, Sandhini and Slama, Katarina and Ray, Alex and others},
  journal={Advances in neural information processing systems},
  volume={35},
  pages={27730--27744},
  year={2022}
}

@inproceedings{pmlr-v178-zimmert22b,
  author = {Zimmert, J. and Lattimore, T.},
  title = {Return of the Bias: Almost Minimax Optimal High Probability Bounds for Adversarial Linear Bandits},
  booktitle = {Proceedings of Thirty Fifth Conference on Learning Theory},
  editor = {Loh, P.-L. and Raginsky, M.},
  series = {Proceedings of Machine Learning Research}, volume = {178},
  publisher = {PMLR}, pages = {3285--3312}, year = {2022}
}

@article{qian2024offline,
  title={Offline Oracle-Efficient Learning for Contextual MDPs via Layerwise Exploration-Exploitation Tradeoff},
  author={Qian, Jian and Hu, Haichen and Simchi-Levi, David},
  journal={arXiv preprint arXiv:2405.17796},
  year={2024}
}

@inproceedings{qiao2022sample,
  author = {Qiao, D. and Yin, M. and Min, M. and Wang, Y.-X.},
  title = {Sample-Efficient Reinforcement Learning with $\log\log(T)$ Switching Cost},
  booktitle = {International Conference on Machine Learning},
  publisher = {PMLR}, pages = {18031--18061}, year = {2022}
}

@misc{qiao2023logarithmic,
  author = {Qiao, D. and Yin, M. and Wang, Y.-X.},
  title = {Logarithmic Switching Cost in Reinforcement Learning beyond Linear {MDPs}},
  year = {2023}, eprint = {2302.12456}, archivePrefix = {arXiv},
  url = {https://arxiv.org/abs/2302.12456}
}

@misc{qin2026taming,
  author = {Qin, H. and Zhang, C.},
  title = {Taming the Monster Every Context: Complexity Measure and Unified Framework for Offline-Oracle Efficient Contextual Bandits},
  year = {2026}, eprint = {2602.09456}, archivePrefix = {arXiv},
  url = {https://arxiv.org/abs/2602.09456}
}

@misc{rakhlin2022mathstat,
  author       = {Sasha Rakhlin},
  title        = {Mathematical Statistics: A Non-Asymptotic Approach},
  year         = {2022},
  howpublished = {\url{https://www.mit.edu/~rakhlin/courses/mathstat/rakhlin_mathstat_sp22.pdf}},
  note         = {Lecture notes, IDS.160, MIT, Spring 2022}
}

@article{simchi2020bypassing,
  title={Bypassing the Monster: A Faster and Simpler Optimal Algorithm for Contextual Bandits under Realizability},
  author={Simchi-Levi, David and Xu, Yunzong},
  journal={Mathematics of Operations Research},
  year={2021}
}

@book{sutton1998reinforcement,
  author = {Sutton, R. S. and Barto, A. G.},
  title = {Reinforcement Learning: An Introduction},
  publisher = {MIT Press}, address = {Cambridge, MA}, year = {1998}
}

@article{usmanova2024log,
  author = {Usmanova, I. and As, Y. and Kamgarpour, M. and Krause, A.},
  title = {Log Barriers for Safe Black-Box Optimization with Application to Safe Reinforcement Learning},
  journal = {Journal of Machine Learning Research},
  volume = {25}, pages = {1--54}, year = {2024}
}

@book{vershynin2018high,
  title={High-dimensional probability: An introduction with applications in data science},
  author={Vershynin, Roman},
  volume={47},
  year={2018},
  publisher={Cambridge university press}
}

@article{wainwright2025wild,
  title={Wild refitting for black box prediction},
  author={Wainwright, Martin J},
  journal={arXiv preprint arXiv:2506.21460},
  year={2025}
}

@article{wang2020reinforcement,
  title={Reinforcement learning with general value function approximation: Provably efficient approach via bounded eluder dimension},
  author={Wang, Ruosong and Salakhutdinov, Russ R and Yang, Lin},
  journal={Advances in Neural Information Processing Systems},
  volume={33},
  pages={6123--6135},
  year={2020}
}

@inproceedings{wang2021provably,
  author = {Wang, T. and Zhou, D. and Gu, Q.},
  title = {Provably Efficient Reinforcement Learning with Linear Function Approximation under Adaptivity Constraints},
  booktitle = {Advances in Neural Information Processing Systems},
  volume = {34}, pages = {13524--13536}, year = {2021}
}

@inproceedings{wei2018more,
  author = {Wei, C.-Y. and Luo, H.},
  title = {More Adaptive Algorithms for Adversarial Bandits},
  booktitle = {Conference on Learning Theory},
  publisher = {PMLR}, pages = {1263--1291}, year = {2018}
}

@inproceedings{xie2023the,
  author = {Xie, T. and Foster, D. J. and Bai, Y. and Jiang, N. and Kakade, S. M.},
  title = {The Role of Coverage in Online Reinforcement Learning},
  booktitle = {The Eleventh International Conference on Learning Representations},
  year = {2023}
}

@article{xu2020upper,
  title={Upper counterfactual confidence bounds: a new optimism principle for contextual bandits},
  author={Xu, Yunbei and Zeevi, Assaf},
  journal={arXiv preprint arXiv:2007.07876},
  year={2020}
}

@inproceedings{zhang2020almost,
  author = {Zhang, Z. and Zhou, Y. and Ji, X.},
  title = {Almost Optimal Model-Free Reinforcement Learning via Reference-Advantage Decomposition},
  booktitle = {Advances in Neural Information Processing Systems},
  volume = {33}, pages = {15198--15207}, year = {2020}
}

@inproceedings{zhang2021exploration,
  author = {Zhang, C. and Cai, Y. and Huang, L. and Li, J.},
  title = {Exploration by Maximizing {R\'enyi} Entropy for Reward-Free {RL} Framework},
  booktitle = {Proceedings of the AAAI Conference on Artificial Intelligence},
  volume = {35}, pages = {10859--10867}, year = {2021}
}

@misc{zhang2024constrained,
  author = {Zhang, B. and Zhang, Y. and Frison, L. and Brox, T. and B{\"o}decker, J.},
  title = {Constrained Reinforcement Learning with Smoothed Log Barrier Function},
  year = {2024}, eprint = {2403.14508}, archivePrefix = {arXiv},
  url = {https://arxiv.org/abs/2403.14508}
}

@article{zhang2025constrained,
  author = {Zhang, B. and Zhang, Y. and Zhu, H. and Yan, S. and Brox, T. and Boedecker, J.},
  title = {Constrained Reinforcement Learning with Smoothed Log Barrier Function},
  journal = {Transactions on Machine Learning Research}, year = {2025}
}

@misc{zheng2019equippingexpertsbanditslongtermmemory,
  author = {Zheng, K. and Luo, H. and Diakonikolas, I. and Wang, L.},
  title = {Equipping Experts/Bandits with Long-Term Memory},
  year = {2019}, eprint = {1905.12950}, archivePrefix = {arXiv},
  url = {https://arxiv.org/abs/1905.12950}
}

@article{hu2026minimax,
  title={Minimax Regret for Sobolev Kernel-Based Reinforcement Learning},
  author={Hu, Haichen and Simchi-Levi, David},
  journal={arXiv preprint arXiv:2610.34045},
  year={2026}
}

\clearpage
\appendix

\section{Mathematical Tools}\label{app-app:math_tools}
\begin{lemma}[MLE for a finite model class \citep{foster2024online}]\label{app-lemma:MLE_finite}
Let $M_*\in\cM$ with $|\cM|<\infty$, and draw $n$ independent labeled trajectories $(\pi_i,Z_i)$ with $\pi_i\sim p\in\Pi_{\mathrm{mix}}$ and $Z_i\sim M_*(\pi_i)$. Define
\[
\hat M\in\argmax_{M\in\cM}\prod_{i=1}^n\PP^{M,\pi_i}(Z_i),
\]
where $\PP^{M,\pi_i}(Z_i)$ denotes the likelihood density under $M(\pi_i)$. For any $\delta\in(0,1/2)$, with probability at least $1-\delta$,
\[
\EE_{\nu\sim p}\left[D_H^2\bigl(\hat M(\nu),M_*(\nu)\bigr)\right]
\le\frac{2\log(|\cM|/\delta)}{n}.
\]
\end{lemma}
\begin{lemma}\citep{foster2021statistical}\label{app-lemma:information_theory}
    Let $\PP$ and $\QQ$ be two distributions on the space $\cX$. Let $h:\cX\rightarrow\RR$ be a function. Then,
    \[
    |\int_{\cX}h(x)d\PP(x)-\int_\cX h(x)d\QQ(x)|\le \sqrt{2\rbr{\int_{\cX}h(x)^2d\PP(x)+\int_\cX h(x)^2d\QQ(x)}D_{\mathrm{H}}^2(\PP,\QQ)}.
    \]
\end{lemma}
\begin{lemma}[Data-Processing Inequality]\label{app-lemma:data_process_ineq}
    Let $\PP_X,\QQ_X\in \cP(\cX)$ and $\PP_{Y|X}$ be a transition kernel from
$(\cX,\cF)$ to $(\cY,\cG)$. Let $\PP_Y,\QQ_Y\in \cP(\cY)$ be the transformation of $\PP_X$ and $\QQ_X$, respectively, when pushed through $\PP_{Y|X}$, i.e., $\PP_X(B) = \int_{\cX} \PP_{Y|X}(B |x)d\PP_{X}(x)$. Then, for any f-divergence, we have
\[
D_f(\PP_X,\QQ_X)\ge D_f(\PP_Y,\QQ_Y).
\]
\end{lemma}
\begin{lemma}[Concavity of the log-determinant]\label{app-lemma:logdet-concave}
Let $\mathbb{S}_{++}^d$ denote the set of $d\times d$ real symmetric positive definite matrices. Then the function
\[
f:\mathbb{S}_{++}^d\to\mathbb{R},\; f(X)=\log\det(X)
\]
is concave. Equivalently, for any $X,Y\in\mathbb{S}_{++}^d$ and any $\lambda\in[0,1]$,
\[
\log\det\bigl(\lambda X+(1-\lambda)Y\bigr)
\ge
\lambda \log\det(X)+(1-\lambda)\log\det(Y).
\]
\end{lemma}

\begin{proof}[Proof of Lemma \ref{app-lemma:logdet-concave}]
For any $X\in\mathbb{S}_{++}^d$ and any symmetric matrix $H$, define
\[
g(t)=\log\det(X+tH),
\]
whenever $X+tH\in\mathbb{S}_{++}^d$. By standard matrix calculus,
\[
g'(t)=\tr\!\bigl((X+tH)^{-1}H\bigr),
\]
and
\[
g''(t)
=
-\tr\!\bigl((X+tH)^{-1}H(X+tH)^{-1}H\bigr)
\le 0.
\]
Hence $g$ is concave in $t$, which implies that $f(X)=\log\det(X)$ is concave on $\mathbb{S}_{++}^d$.
\end{proof}

\begin{lemma}\label{app-lemma:sub_prob_covariance}
    Let $(\Omega, \mathcal{F})$ be a measurable space and $\mu$ be a positive measure on $\Omega$ such that its total mass is bounded by $1$, i.e., $m = \int_{\Omega} d\mu \le 1$. Let $X: \Omega \to \mathbb{R}^d$ be a $\mu$-integrable vector-valued function. Then, the following matrix inequality holds in the positive semi-definite sense:
    \[
    \int_{\Omega} X(\omega) X(\omega)^T d\mu(\omega) \succeq \left( \int_{\Omega} X(\omega) d\mu(\omega) \right) \left( \int_{\Omega} X(\omega) d\mu(\omega) \right)^T.
    \]
\end{lemma}

\begin{proof}[Proof of Lemma \ref{app-lemma:sub_prob_covariance}]
    To prove that the difference $\int_{\Omega} X X^T d\mu - (\int_{\Omega} X d\mu)(\int_{\Omega} X d\mu)^T$ is positive semi-definite, we need to show that for any arbitrary vector $v \in \mathbb{R}^d$, the corresponding quadratic form is non-negative:
    \[
    v^T \left( \int_{\Omega} X(\omega) X(\omega)^T d\mu(\omega) \right) v \ge v^T \left( \int_{\Omega} X(\omega) d\mu(\omega) \right) \left( \int_{\Omega} X(\omega) d\mu(\omega) \right)^T v.
    \]
    By the linearity of the integral, we can bring the constant vector $v$ inside the integral. Let us define a scalar function $f(\omega) = v^T X(\omega)$. The inequality can then be equivalently rewritten in terms of this scalar function as:
    \[
    \int_{\Omega} f(\omega)^2 d\mu(\omega) \ge \left( \int_{\Omega} f(\omega) d\mu(\omega) \right)^2.
    \]
    We now apply the Cauchy-Schwarz inequality for integrals with respect to the measure $\mu$. For the functions $f(\omega)$ and the constant function $1$, we have:
    \[
    \left( \int_{\Omega} f(\omega) \cdot 1 d\mu(\omega) \right)^2 \le \left( \int_{\Omega} f(\omega)^2 d\mu(\omega) \right) \left( \int_{\Omega} 1^2 d\mu(\omega) \right).
    \]
    Observe that $\int_{\Omega} 1 d\mu(\omega) = \mu(\Omega) = m$. Substituting this back yields:
    \[
    \left( \int_{\Omega} f(\omega) d\mu(\omega) \right)^2 \le m \int_{\Omega} f(\omega)^2 d\mu(\omega).
    \]
    Since the integral of a squared real function is always non-negative, $\int_{\Omega} f(\omega)^2 d\mu(\omega) \ge 0$. Given the assumption that the total measure $m \le 1$, we can further upper bound the right-hand side:
    \[
    m \int_{\Omega} f(\omega)^2 d\mu(\omega) \le 1 \cdot \int_{\Omega} f(\omega)^2 d\mu(\omega).
    \]
    Chaining the inequalities together, we obtain:
    \[
    \left( \int_{\Omega} f(\omega) d\mu(\omega) \right)^2 \le \int_{\Omega} f(\omega)^2 d\mu(\omega).
    \]
    Because this holds for any vector $v \in \mathbb{R}^d$, the original matrix inequality is satisfied, concluding the proof.
\end{proof}

\begin{lemma}[Local simulation lemma \citep{foster2021statistical}]\label{app-lemma:local_simulation_lemma}
For any pair of MDPs $M = (P^M, R^M)$ and $\bar{M} = (P^{\bar{M}}, R^{\bar{M}})$
with the same initial state distribution and $\sum_{h=1}^H r_h \in [0,1]$,
\begin{align}
&V_M^1(\pi) - V_{\bar{M}}^1(\pi)\nonumber\\
=& \sum_{h=1}^H \mathbb{E}^{\bar{M},\pi}
\left[
\bigl[(P_h^M - P_h^{\bar{M}}) V_{h+1}^{M,\pi}\bigr](s_h,a_h)
\right] \notag   + \sum_{h=1}^H \mathbb{E}^{\bar{M},\pi}
\left[
\mathbb{E}_{r_h \sim R_h^M(s_h,a_h)}[r_h]
-
\mathbb{E}_{r_h \sim R_h^{\bar{M}}(s_h,a_h)}[r_h]
\right] \\
\le&
\sum_{h=1}^H \mathbb{E}^{\bar{M},\pi}
\left[
D_{\mathrm{TV}}\bigl(P_h^M(s_h,a_h), P_h^{\bar{M}}(s_h,a_h)\bigr)
+
D_{\mathrm{TV}}\bigl(R_h^M(s_h,a_h), R_h^{\bar{M}}(s_h,a_h)\bigr)
\right].
\end{align}
\end{lemma}
\begin{lemma}\label{app-lemma:iterative_Xu_QianJian}[\cite{simchi2020bypassing,qian2024offline}]\label{app-lemma:iterative_value_func}
    Let $\varepsilon_1,\cdots,\varepsilon_N$ be $N$ positive values. Suppose for $m>0$ and an arbitrary policy $\pi$, we have
    \[
    |\hat{V}^1_m(s^1,\pi)-V_*^1(s^1,\pi)|\le \frac{1}{20}\hat{\mathrm{reg}}_{m-1}(\pi)+\varepsilon_m.
    \]
    Then for any $m>0$,
    \[
    \mathrm{reg}(\pi)\le \frac{10}{9}\hat{\mathrm{reg}}_m(\pi)+\delta_m,
    \]
    \[
    \hat{\mathrm{reg}}_m(\pi)\le \frac{9}{8}\mathrm{reg}(\pi)+\delta_m,
    \]
    where $\delta_1=2\varepsilon_1+\frac{1}{10}$ and $\delta_m=\frac{1}{9}\delta_{m-1}+\frac{20}{9}\varepsilon_m$ for any $m\ge 2$.
\end{lemma}

\begin{lemma}[Conditional least squares for a finite model class]\label{app-lemma:conditional_LSE_finite}
Let $\cM$ be finite and contain $M_*$. For i.i.d. trajectories with base policies drawn from $p$, minimize the empirical conditional loss
\[
\ell_M(Z)=\sum_{h=1}^H\left[
\|P_M^h(\cdot\mid s^h,a^h)-e_{s^{h+1}}\|_{L^2}^2
+(r_M^h(s^h,a^h)-R^h)^2\right],
\]
where $e_{s'}$ is the unit mass at $s'$. Assume rewards and predictions lie in $[0,1]$. With probability at least $1-\delta$, its minimizer $\hat M$ satisfies
\[
\sum_{h=1}^H\EE_{\pi\sim p}\EE^{M_*,\pi}\left[
\|P_{\hat M}^h(\cdot\mid s^h,a^h)-P_*^h(\cdot\mid s^h,a^h)\|_{L^2}^2
+|r_{\hat M}^h(s^h,a^h)-r_*^h(s^h,a^h)|^2\right]
\le\frac{40H\log(|\cM|/\delta)}{n}.
\]
In particular this bounds each layer's risk in Definition~\ref{def:offline_oracle}. This is one offline regression call using the same trajectories.
\end{lemma}
\begin{proof}
For a fixed $M$, put $D_M(Z)=\ell_M(Z)-\ell_{M_*}(Z)$ and $R_M=\EE[D_M(Z)]$. Conditional expectation at each state-action pair shows that $R_M$ equals the sum of squared prediction errors in the claim: the cross terms have conditional mean zero. A probability vector is within $\sqrt2$ of any unit mass in $L^2$, so the difference of the two transition losses has absolute value at most $2\sqrt2\|P_M^h-P_*^h\|_{L^2}$. The reward-loss difference has absolute value at most $2|r_M^h-r_*^h|$. Cauchy--Schwarz across the $H$ layers therefore gives
\[
\EE[D_M(Z)^2]\le16H R_M,\; |D_M(Z)|\le3H.
\]
For the lower tail, the centered variable $R_M-D_M(Z)$ is bounded above by $6H$. Bernstein's inequality and a union bound over $M\in\cM$ imply, with $u=\log(|\cM|/\delta)/n$,
\[
\frac1n\sum_{i=1}^nD_M(Z_i)
\ge R_M-\sqrt{32H R_M u}-2Hu
\]
simultaneously for all $M$. Empirical optimality makes the left side nonpositive at $\hat M$. Using $\sqrt{32H R_Mu}\le R_M/2+16Hu$ gives $R_{\hat M}\le36Hu\le40Hu$, as required.
\end{proof}

\begin{lemma}[Lagged calibration]\label{app-lemma:lagged_calibration}
Suppose all true and estimated total values lie in $[0,1]$. Let $a_0=1$ and $0<a_m\le a_{m-1}\le1$, and set $\eta_m=(100a_{m-1})^{-1}$. If, uniformly over $p\in\Pi_{\mathrm{mix}}$,
\[
|\hat V_m^1(p)-V_*^1(p)|\le a_m\bigl(1+\eta_m\hatreg_{m-1}(p)\bigr),
\]
then, for every $m\ge0$ and $p\in\Pi_{\mathrm{mix}}$,
\[
\hatreg_m(p)\le2\mathrm{reg}(p)+10a_m,\;
\mathrm{reg}(p)\le2\hatreg_m(p)+10a_m.
\]
\end{lemma}
\begin{proof}
Both claims at $m=0$ follow from the value range and $a_0=1$. Assume the claims at $m-1$, put $e_m(p)=\hat V_m^1(p)-V_*^1(p)$ and $\alpha_m=a_m/(100a_{m-1})\le1/100$, and let $p_m$ maximize $\hat V_m^1$. Optimality and the inductive upper bound imply
\[
\mathrm{reg}(p_m)
\le e_m(p_m)-e_m(\delta_{\pi_*})
\le2a_m+\alpha_m\bigl(2\mathrm{reg}(p_m)+20a_{m-1}\bigr),
\]
so $\mathrm{reg}(p_m)\le3a_m$. For arbitrary $p$, use
\[
\hatreg_m(p)=\mathrm{reg}(p)-\mathrm{reg}(p_m)+e_m(p_m)-e_m(p)
\]
and the same inductive bound to obtain
\[
0.98\mathrm{reg}(p)-5.26a_m
\le\hatreg_m(p)
\le1.02\mathrm{reg}(p)+2.26a_m.
\]
These inequalities imply both claims, using $\hatreg_m(p)\ge0$ for the second one. The factors $a_m/a_{m-1}$ retain the current error scale even when epoch lengths grow faster than geometrically.
\end{proof}

\section{Proofs in Subsection \ref{subsec:theory_perepoch_tabular}}
First, we show that for any $h,m $, the trusted occupancy measure set $\cbr{\tilde{d}^h_m(\cdot;p) :p\in\Pi_{\mathrm{mix}}}$ is a convex set.
\begin{lemma}\label{app-lemma:trusted_occupancy_convex_tabular}
    Under any $h,m $, for any $0<\lambda<1$ and two arbitrary mixtures $p_1,p_2\in\Pi_{\mathrm{mix}}$, the meta-policy $p'_{1,2}=\lambda p_1+(1-\lambda)p_2\in\Pi_{\mathrm{mix}}$ satisfies 
    $$\lambda\tilde{d}^{h}_m(s';p_1) +(1-\lambda)\tilde{d}^h_m(s';p_2) =\tilde{d}^h_m(s';p'_{1,2}) ,\ \forall s'\in\cS,$$ 
    and
    \[
    \lambda\tilde{d}^{h}_m(s',a';p_1) +(1-\lambda)\tilde{d}^h_m(s',a';p_2) =\tilde{d}^h_m(s',a';p'_{1,2}) ,\ \forall (s',a')\in\cS\times\cA.
    \]

    The same $p'_{1,2}$ also satisfies, for every model $M$,
    \[
    V_M^1(p'_{1,2})=\lambda V_M^1(p_1)+(1-\lambda)V_M^1(p_2).
    \]
\end{lemma}
\begin{proof}[Proof of Lemma \ref{app-lemma:trusted_occupancy_convex_tabular}]
Take $p'_{1,2}=\lambda p_1+(1-\lambda)p_2\in\Pi_{\mathrm{mix}}$, where the mixture is drawn once at the beginning of a trajectory. Conditional on the selected component, execute its complete Markov policy. For the fixed base-policy truncation rules, the law of total expectation gives both occupancy identities at every layer. The same law gives the value identity for every $M$, regardless of which model defines the trusted occupancies. In particular, no Markovization under a current model is used to evaluate the value under a previous model.
\end{proof}
\begin{proof}[Proof of Lemma \ref{lemma:hatpi_property_tabular}]
By Lemma \ref{app-lemma:trusted_occupancy_convex_tabular}, $\lambda p_m^h+(1-\lambda)p'\in\Pi_{\mathrm{mix}}$ is an episode-level mixture whose current trusted occupancy and previous-model value are both affine in $\lambda$. We evaluate the planning objective along this feasible path for $\lambda\in[0,1]$.
\begin{align*}
    L(\lambda,p^h_m,p')=&\hat{V}^1_{m-1}(\lambda p^h_m+(1-\lambda)p')+\frac{1}{\eta_m}\sum_{s,a}\log(\lambda\tilde{d}^h_m(s,a;p^h_m)+(1-\lambda)\tilde{d}^h_m(s,a;p')+\beta_m).
\end{align*}
By the properties of the logarithm, we know $\frac{1}{\eta_m}\sum_{s,a}\log(\lambda\tilde{d}^h_m(s,a;p^h_m)+(1-\lambda)\tilde{d}^h_m(s,a;p')+\beta_m)$ is concave over $\lambda$. Meanwhile, $\hat{V}^1_{m-1}(\lambda p^h_m+(1-\lambda)p')$ is a linear function over $\lambda$. Therefore, we have $L(\lambda,p^h_m,p')$ is a concave function with respect to $\lambda$. Since $L(\lambda,p^h_m,p')$ is maximized when $\lambda=1$, $L(\lambda,p^h_m,p')$ increases monotonically as we increase $\lambda$ from $0$ to $1$. Thus,
\[
\frac{dL}{d\lambda}\ge 0,\ \lambda\in[0,1].
\]
Compute the derivative at $\lambda=1$ and we obtain
\begin{align*}
    \eta_m(\hatreg_{m-1}(p')-\hatreg_{m-1}(p^h_m))+\sum_{s,a}\rbr{\frac{\tilde{d}^h_m(s,a;p^h_m)}{\tilde{d}^h_m(s,a;p^h_m)+\beta_m}-\frac{\tilde{d}^h_m(s,a;p')}{\tilde{d}^h_m(s,a;p^h_m)+\beta_m}}\ge 0.
\end{align*}
Therefore, for any fixed $s,a$, we have
\[
\tilde{d}^h_m(s,a;p')\le (\tilde{d}^h_m(s,a;p^h_m)+\beta_m)(SA+\eta_m\hatreg_{m-1}(p')).
\]
This completes the proof.
\end{proof}

\begin{proof}[Proof of Lemma \ref{lemma:tilded_boundby_hatd_tabular}]
For each base policy, decompose its occupancy according to the first discarded transition. Conditional on that transition, the probability of the terminal state-action event under consideration is at most one. Average this inequality over the components of $p$ to obtain
\[
0\le\hat d_m^h(s,a;p)-\tilde d_m^h(s,a;p)
\le\sum_{k<h}\sum_{(u,b,v)\notin\tilde{\cT}_m^{k+1}}
\tilde d_m^k(u,b;p)\hat P_m^k(v\mid u,b).
\]
Write $R=\hatreg_{m-1}(p)$. The coverage bound in Lemma~\ref{lemma:hatpi_property_tabular} bounds the right side by
\[
(SA+\eta_m R)\sum_{k<h}\sum_{(u,b,v)\notin\tilde{\cT}_m^{k+1}}
\bigl(\tilde d_m^k(u,b;p_m^k)+\beta_m\bigr)\hat P_m^k(v\mid u,b).
\]
Each untrusted triple has reference weight less than $1/\zeta_m$; there are at most $S^2A$ triples per layer. Also $\sum_v\hat P_m^k(v\mid u,b)=1$. Consequently this expression is at most
\[
(SA+\eta_m R)\left(\frac{HS^2A}{\zeta_m}+HSA\beta_m\right).
\]
Expanding and enlarging the dimension factors using $S,A\ge1$ gives the stated bound. All averaging is performed before the coverage bound, so no model-dependent effective Markov policy is introduced.
\end{proof}

\begin{proof}[Proof of Lemma~\ref{lemma:tilde_dP<dP*_tabular}]
Set $c_h=(1+1/H)^{2(h-1)}$. We prove
\[
\tilde d_m^h(s,a;\nu)\le c_h d_m^h(s,a;\nu)
\]
simultaneously for all base policies $\nu\in\Pi_{RNS}$. Averaging this
inequality over $\nu\sim p$ gives the same inequality for every meta-policy
$p$, including the executed $p_m^h$. The base case $h=1$ follows from the
common initialization.

Suppose the claim holds at layer $h<H$. Fix a retained triple
$(s,a,s')\in\tilde{\cT}_m^{h+1}$, and abbreviate
\[
\hat q=\hat P_m^h(s'\mid s,a),\ 
q=P_*^h(s'\mid s,a),\ 
w=\tilde d_m^h(s,a;p_m^h).
\]
If $\hat q>(1+1/H)^2q$, then
$(\sqrt{\hat q}-\sqrt q)^2>\hat q/(H+1)^2$. Since the squared
Hellinger distance includes this coordinate, the inductive bound averaged
over $p_m^h$ gives
\begin{align*}
w\hat q
&<(H+1)^2w
D_H^2\bigl(\hat P_m^h(\cdot\mid s,a),P_*^h(\cdot\mid s,a)\bigr)\\
&\le c_h(H+1)^2\sum_{u,b}d_m^h(u,b;p_m^h)
D_H^2\bigl(\hat P_m^h(\cdot\mid u,b),P_*^h(\cdot\mid u,b)\bigr)\\
&\le e^2(H+1)^2\cE_m<\zeta_m^{-1},
\end{align*}
contradicting the definition of a retained triple. Thus
\[
\hat P_m^h(s'\mid s,a)\le(1+1/H)^2P_*^h(s'\mid s,a)
\ \text{on every }\tilde{\cT}_m^{h+1}.
\]
This comparison depends only on the common trusted set and the two
kernels. For any base policy $\nu$, its occupancy recursion now yields
\begin{align*}
\tilde d_m^{h+1}(s',a';\nu)
&=\nu^{h+1}(a'\mid s')
\sum_{(s,a):(s,a,s')\in\tilde{\cT}_m^{h+1}}
\tilde d_m^h(s,a;\nu)\hat P_m^h(s'\mid s,a)\\
&\le c_h(1+1/H)^2\nu^{h+1}(a'\mid s')
\sum_{(s,a):(s,a,s')\in\tilde{\cT}_m^{h+1}}
d_m^h(s,a;\nu)P_*^h(s'\mid s,a)\\
&=c_{h+1}d_m^{h+1}(s',a';\nu).
\end{align*}
This completes the induction. Finally, for
$p=\sum_jw_j\delta_{\nu_j}$,
\[
\tilde d_m^h(s,a;p)
=\sum_jw_j\tilde d_m^h(s,a;\nu_j)
\le c_h\sum_jw_jd_m^h(s,a;\nu_j)
=c_hd_m^h(s,a;p).
\]
Since $c_h\le(1+1/H)^{2h}$, this proves every asserted conclusion.
\end{proof}

\begin{proof}[Proof of Lemma \ref{lemma:perepoch_valuefunc_error_tabular}]
For a base policy, apply the local simulation identity while retaining its continuation value, and bound that value by $1$. Then average over the components if $p$ is a mixture. Put $R=\hatreg_{m-1}(p)$ and
\[
B(p):=\frac{H^2S^4A^4}{\zeta_m}+\frac{\eta_mH^2S^3A^3R}{\zeta_m}
+H^2S^3A^3(1+\eta_mR)\beta_m.
\]
Lemma \ref{lemma:tilded_boundby_hatd_tabular}, summed over states, actions and layers, bounds the total omitted occupancy mass by $B(p)$. The transition and reward contributions together are at most twice this mass. On retained trajectories, Lemma \ref{lemma:hatpi_property_tabular} gives the factor $SA+\eta_mR$, and Lemma \ref{lemma:tilde_dP<dP*_tabular} replaces the reference trusted occupancy by at most $e^2$ times its true observable occupancy. The latter is dominated by the full true occupancy. For each of the transition and reward kernels, Cauchy--Schwarz and $\|P-Q\|_1\le2D_H(P,Q)$ give average $L^1$ error at most $2\sqrt{\cE_m}$ on the layerwise oracle event. The $\beta_m$ part contributes at most $4HSA\beta_m(SA+\eta_mR)$. Consequently,
\[
|\hat V_m^1(p)-V_*^1(p)|
\le2B(p)+4e^2H(SA+\eta_mR)\sqrt{\cE_m}
+4HSA\beta_m(SA+\eta_mR).
\]
Since $H\ge1$, this is bounded by the displayed expression in the lemma. All uses of occupancy recursion above are componentwise; the final bounds depend only on averaged occupancies and the affine pseudo-regret.

\end{proof}
\begin{proof}[Proof of Lemma~\ref{lemma:pseudo_regret_tabular}]
Since $p_m^h$ maximizes the segment-$h$ objective over $\Pi_{\mathrm{mix}}$, comparison with the estimated-value benchmark $p_{m-1}$ gives
\begin{align*}
\eta_m\hatreg_{m-1}(p_m^h)
&\le\sum_{s,a}\log\bigl(\tilde d_m^h(s,a;p_m^h)+\beta_m\bigr)
-\sum_{s,a}\log\bigl(\tilde d_m^h(s,a;p_{m-1})+\beta_m\bigr)\\
&\le SA\log(\beta_m+1/(SA))-SA\log\beta_m.
\end{align*}
For the last inequality, we used $\sum_{s,a}\tilde d_m^h(s,a;p_m^h)\le1$, Jensen's inequality for the first sum, and nonnegativity of the benchmark occupancy for the second. Dividing by $\eta_m$ proves the claim.
\end{proof}

\section{Proofs in Subsection \ref{subsec:theory_regret_analysis_tabular}}

\begin{proof}[Proof of Theorem \ref{thm:main_regret_tabular}]
Take a union bound over the $H$ oracle calls in each epoch at levels $\delta_m$; their sum is less than $\delta$. To extract a layerwise Hellinger bound from a trajectory bound, write $X$ for the layer's state-action pair and $Y$ for its next state or reward, and compare the true joint law $P_{XY}$ with the hybrid $P_XQ_{Y\mid X}$. The triangle inequality and data processing give
\[
\bigl(\EE_{X\sim P_X}D_H^2(P_{Y\mid X},Q_{Y\mid X})\bigr)^{1/2}
\le D_H(P_{XY},Q_{XY})+D_H(P_X,Q_X)
\le2D_H(P_{\rm traj},Q_{\rm traj}).
\]
Apply this to each sampled base-policy label and average. Thus $\cE_m\ge4\cE_{\cM}^{\delta_m}(n_m)$ bounds the transition risks for $h<H$ and reward risks for $h\le H$. The last transition has zero continuation value and is not needed. True observable occupancies are smaller than the full true occupancies, so the same bound applies to them.

The chosen $\zeta_m$ satisfies Lemma \ref{lemma:tilde_dP<dP*_tabular} at every epoch. When $a_m=1$, the value range gives error at most $1$. When $a_m<1$, we have $\cE_m<1$ and, since $a_{m-1}\ge a_m=C_{\rm tab}\sqrt{\cE_m}$,
\[
\eta_m\cE_m\le\frac{\sqrt{\cE_m}}{100C_{\rm tab}},\;
\beta_m\le\sqrt{\cE_m}.
\]
Substituting the parameters into Lemma \ref{lemma:perepoch_valuefunc_error_tabular} bounds its constant coefficient by $a_m$ and its pseudo-regret coefficient by $\eta_ma_m$. For example, if $A_m=2H^2S^4A^4/\zeta_m+2H^2S^3A^3\beta_m+4e^2H^2S^2A^2\beta_m+4e^2HSA\sqrt{\cE_m}$, then all dimension factors are dominated by $C_{\rm tab}$ and $A_m\le C_{\rm tab}\sqrt{\cE_m}=a_m$; dividing the pseudo-regret coefficient by $\eta_m$ yields smaller dimension factors. Therefore, uniformly over $p$,
\[
|\hat V_m^1(p)-V_*^1(p)|\le a_m(1+\eta_m\hatreg_{m-1}(p)).
\]
Lemma \ref{app-lemma:lagged_calibration} implies
\[
\mathrm{reg}(p_m^h)
\le2\hatreg_{m-1}(p_m^h)+10a_{m-1}
\le210SA L_m a_{m-1},
\]
where the final step uses Lemma \ref{lemma:pseudo_regret_tabular} and $\eta_m^{-1}=100a_{m-1}$. Regret is also at most $1$. Summing over the $Hn_m$ trajectories of each epoch proves the claim.

\end{proof}

\begin{proof}[Proof of Corollary \ref{cor:reg_known_T_tabular}]
Put $\Lambda=\log(4H|\cM|(N+1)^2/\delta)$. The finite-class maximum-likelihood bound follows from the affinity test $\PP_*(L_M\ge L_{M_*})\le(1-D_H^2(M,M_*)/2)^{n_m}$ and a union bound; applied to the joint law of the sampled base-policy label and its trajectory, this gives averaged trajectory Hellinger risk at most $2\log(|\cM|/\delta_m)/n_m$. Thus the chosen constant $80$ includes the factor $4$ for conditional layer extraction.

For $m\ge2$, $a_{m-1}\le C_{\rm tab}\sqrt{80\Lambda/n_{m-1}}$, and $L_m\lesssim1+\log(1+T)$. Before the final endpoint, $n_{m-1}\ge\tau_{m-1}/2$ and
\[
\frac{n_m}{\sqrt{n_{m-1}}}
\le\frac{2B^{1-2^{-m}}}{B^{(1-2^{1-m})/2}}=2\sqrt B.
\]
The appended final endpoint satisfies $\tau_{N-1}>B/8$, because the next proposed endpoint $\sqrt{2B\tau_{N-1}}$ exceeded $B/2$. Hence $n_{N-1}>B/16$ and $n_N/\sqrt{n_{N-1}}\le4\sqrt B$. Integer rounding only changes these constants. The first epoch contributes at most $Hn_1\lesssim H\sqrt B$. Theorem \ref{thm:main_regret_tabular}, summed over $N$ epochs, now gives the claimed bound and oracle counts.

\end{proof}
\begin{proof}[Proof of Corollary \ref{cor:reg_unknown_T_tabular}]
For $m\ge2$, $n_m=2^{m-1}$, while $n_1=2$. One explicit positive nonincreasing majorant valid also for these first two equal-sized epochs is
\[
\cE_m=\frac{160\log(4H|\cM|(m+2)^2/\delta)}{2^{m-1}}.
\]
It dominates four times the finite-class maximum-likelihood trajectory bound and has the asserted order. At evaluation horizon $T$, the planned length of the last started epoch is at most a constant times $T/H$. For $m\ge3$, $n_m/\sqrt{n_{m-1}}=\sqrt{2n_m}$, and the geometric sum of these quantities is $O(\sqrt{T/H})$. Use $a_{m-1}\le C_{\rm tab}\sqrt{\cE_{m-1}}$ and $L_m\lesssim1+\log(1+T)$ in Theorem \ref{thm:main_regret_tabular}. The first two epochs and an incomplete final epoch satisfy the same bound after adjusting its absolute constant. This proves the regret estimate; the number of started epochs is $O(\log(T+2))$.

\end{proof}

\section{Proofs in Subsection \ref{subsec:theory_per_epoch_linear}}\label{app-app:proofs_subsec_theory_per_epoch}
\begin{lemma}\label{app-lemma_trusted_occupancy_convex}
    Under any $h,m $, for any $0<\lambda<1$ and two arbitrary mixtures $p_1,p_2\in\Pi_{\mathrm{mix}}$, the meta-policy $p'_{1,2}=\lambda p_1+(1-\lambda)p_2\in\Pi_{\mathrm{mix}}$ satisfies 
    $$\lambda\tilde{d}^{h}_m(s';p_1) +(1-\lambda)\tilde{d}^h_m(s';p_2) =\tilde{d}^h_m(s';p'_{1,2}) ,\ \forall s'\in\cS,$$ 
    and
    \[
    \lambda\tilde{d}^{h}_m(s',a';p_1) +(1-\lambda)\tilde{d}^h_m(s',a';p_2) =\tilde{d}^h_m(s',a';p'_{1,2}) ,\ \forall (s',a')\in\cS\times\cA.
    \]

    The same $p'_{1,2}$ also satisfies, for every model $M$,
    \[
    V_M^1(p'_{1,2})=\lambda V_M^1(p_1)+(1-\lambda)V_M^1(p_2).
    \]
\end{lemma}
\begin{proof}[Proof of Lemma \ref{app-lemma_trusted_occupancy_convex}]
Take $p'_{1,2}=\lambda p_1+(1-\lambda)p_2\in\Pi_{\mathrm{mix}}$, where the mixture is drawn once at the beginning of a trajectory. Conditional on the selected component, execute its complete Markov policy. For the fixed base-policy truncation rules, the law of total expectation gives both occupancy identities at every layer. The same law gives the value identity for every $M$, regardless of which model defines the trusted occupancies. In particular, no Markovization under a current model is used to evaluate the value under a previous model.
\end{proof}

\begin{proof}[Proof of Lemma \ref{lemma:hatpi_property_linear}]
    For any $p\in\Pi_{\mathrm{mix}}$ and any $0<\lambda<1$, we define the following function with respect to $\lambda$.
    \[
    L(\lambda;p^h_m,p) :=\hat{V}^1_{m-1}(\lambda p^h_m +(1-\lambda)p) +\frac{1}{\eta_m}\sum_{j=1}^{h}\log(\det(\lambda\tilde{K}^j_m(p^h_m) +(1-\lambda)\tilde{K}^j_m(p)+\beta_mI_d)).
    \]
    By Lemma \ref{app-lemma_trusted_occupancy_convex}, the function $L$ is well-defined for any $\lambda\in[0,1]$. By episode-level mixing and the law of total expectation, for any model $M$ and any two mixtures $p_1,p_2\in\Pi_{\mathrm{mix}}$, we have $V_M^1(\lambda p_1+(1-\lambda)p_2) =\lambda V_M^1(p_1) +(1-\lambda)V_M^1(p_2) $. The value function is linear in $\lambda$. 

    Meanwhile, by Lemma \ref{app-lemma:logdet-concave}, the function $\log(\det(\tilde{K}^j_m(\lambda p^h_m +(1-\lambda)p) +\beta_mI_d))$ is concave in $\lambda$. Therefore, $L(\lambda;p^h_m,p) $ is concave in $\lambda$. Since $p^h_m $ is the maximizer of this function, by the first order optimality condition, we have
    \[
    \frac{\partial L(\lambda;p^h_m,p) }{\partial\lambda}\Big|_{\lambda=1}\ge 0.
    \]
    Differentiating with respect to $\lambda$, we obtain
    \[
    \hat{V}_{m-1}^1(p^h_m ) -\hat{V}_{m-1}^1(p) +\frac{1}{\eta_m}\sum_{j=1}^{h}\tr\rbr{\rbr{\tilde{K}^j_m(p^h_m ) +\beta_mI_d}^{-1}\rbr{\tilde{K}^j_m(p^h_m ) -\tilde{K}^j_m(p) }}\ge 0.
    \]
    Rearranging, we have
    \begin{align*}
        \sum_{j=1}^{h}\tr((\tilde{K}^j_m(p^h_m ) +\beta_mI_d)^{-1}\tilde{K}^j_m(p) )\le& \sum_{j=1}^{h}\tr((\tilde{K}^j_m(p^h_m ) +\beta_mI_d)^{-1}\tilde{K}^j_m(p^h_m ) )\\
        &+\eta_m\rbr{\hat{V}_{m-1}^1(p^h_m ) -\hat{V}_{m-1}^1(p) }\\
        \le& hd+\eta_m(\hat{\mathrm{reg}}_{m-1}(p) -\hat{\mathrm{reg}}_{m-1}(p^h_m ) ).
    \end{align*}
    In the last inequality, we use $\tr((K+\beta_m I_d)^{-1}K)=\sum_{i=1}^d\lambda_i(K)/(\lambda_i(K)+\beta_m)\le d$ for every positive semidefinite $K$.

    This completes the proof.
\end{proof}

\begin{proof}[Proof of Lemma \ref{lemma:trusted_occupancy_boundby_estimated_occupancy}]
First fix a base Markov policy $\rho$. Decompose the difference between its full estimated occupancy and trusted occupancy according to the first discarded state. Since the probability of any subsequent event is at most one,
\[
 \int_{s,a}\bigl(\hat d_m^h(s,a;\rho)-\tilde d_m^h(s,a;\rho)\bigr)
 \le\sum_{j=2}^h\int_{s'\notin\tilde{\cT}_m^j(\rho)}
                     \tilde p_m^j(s';\rho).
\]
The trusted state occupancy has total mass at most one, and $\|\phi_{j-1}\|_2\le1$, so
\[
 0\le\tilde p_m^j(s';\rho)
 \le\|\hat\mu_{m,j-1}^j(s')\|_2,
 \ 
 \int_{s'}\sqrt{\tilde p_m^j(s';\rho)}\le c_{\cM}.
\]
On an omitted state, Definition~\ref{def:trusted_occupancy_measure_linear} gives
\[
 \tilde p_m^j(s';\rho)<\zeta_m^{-1}\sqrt{B_m(\rho)}.
\]
It follows that
\begin{align*}
 \int_{s'\notin\tilde{\cT}_m^j(\rho)}\tilde p_m^j(s';\rho)
 &\le\zeta_m^{-1/2}B_m(\rho)^{1/4}
                  \int_{s'}\sqrt{\tilde p_m^j(s';\rho)}\le c_{\cM}\zeta_m^{-1/2}B_m(\rho)^{1/4}.
\end{align*}
Summing over $j$ proves the base-policy bound. For an episode mixture $p$, average the first-exit inequality over $\rho\sim p$ and apply Jensen's inequality:
\[
 \EE_{\rho\sim p}[B_m(\rho)^{1/4}]
 \le\left(\EE_{\rho\sim p}[B_m(\rho)]\right)^{1/4}
 =B_m(p)^{1/4}.
\]
This proves the stated result.
\end{proof}

\begin{proof}[Proof of Lemma \ref{lemma:tilde_dP<dP*_linear}]
We prove a pointwise comparison simultaneously for all base Markov policies by induction on the layer. Set $\kappa_j=(1+1/H)^{j-1}$. At layer one, the two occupancies agree. Assume that
\[
 \tilde d_m^k(s,a;\rho)\le\kappa_kd_m^k(s,a;\rho)
\]
for every base policy $\rho$. Averaging this inequality over the episode mixture $p_m^k$ gives
\[
 \tilde K_m^k(p_m^k)
 \preceq\kappa_kK_m^k(p_m^k)
 \preceq e^2K_{*,\mathrm{full}}^k(p_m^k),
\]
where $K_{*,\mathrm{full}}^k(p):=\int_{s,a}\phi_k(s,a)\phi_k(s,a)^T d_{M_*}^k(s,a;p)$ is the covariance under the full true-model occupancy. The second inequality uses $\kappa_k\le e<e^2$ and the fact that a truncated occupancy is dominated by the corresponding full occupancy, separately for each support policy.

Write
\[
 U_k=\tilde K_m^k(p_m^k)+\beta_mI_d,
 \ 
 \Delta\mu_k(s')=\hat\mu_{m,k}^{k+1}(s')-\mu_*^{k+1}(s').
\]
The oracle event and the covariance comparison imply, for every $s'$,
\begin{align*}
 \|\Delta\mu_k(s')\|_{U_k}^2
 &\le e^2\EE^{M_*,p_m^k}
       \left[\bigl|\hat P_m^k(s'\mid s^k,a^k)
                         -P_*^k(s'\mid s^k,a^k)\bigr|^2\right]
       +\beta_m\|\Delta\mu_k(s')\|_2^2\le e^2E_m+4c_{\cM}^4\beta_m=v_m.
\end{align*}
Here $\|\mu(s')\|_2\le c_{\cM}^2$ follows from the normalization condition. For a fixed base policy $\rho$, the trace guarantee in Lemma~\ref{lemma:hatpi_property_linear}, applied at segment $k$ with comparator equal to the Dirac mixture at $\rho$, gives
\[
 \tr\left(U_k^{-1}\tilde K_m^k(\rho)\right)
 \le kd+\eta_m\hatreg_{m-1}(\rho)\le B_m(\rho).
\]
Since $\tilde d_m^k(\cdot;\rho)$ is a sub-probability measure, the sub-probability covariance inequality and Cauchy--Schwarz yield
\begin{align*}
 &\left|\int_{s,a}\tilde d_m^k(s,a;\rho)
           \bigl(\hat P_m^k-P_*^k\bigr)(s'\mid s,a)\right|\le\|\Delta\mu_k(s')\|_{U_k}
       \left(\tr\left(U_k^{-1}\tilde K_m^k(\rho)\right)\right)^{1/2}
 \le\sqrt{B_m(\rho)v_m}.
\end{align*}
For $s'\in\tilde{\cT}_m^{k+1}(\rho)$, its predicted arrival mass obeys
\[
 \tilde p_m^{k+1}(s';\rho)
 \ge(H+1)\sqrt{B_m(\rho)v_m}.
\]
Consequently,
\begin{align*}
 \int_{s,a}\tilde d_m^k(s,a;\rho)P_*^k(s'\mid s,a)
 &\ge\tilde p_m^{k+1}(s';\rho)-\sqrt{B_m(\rho)v_m}\ge\frac H{H+1}\tilde p_m^{k+1}(s';\rho).
\end{align*}
Combining this inequality with the induction hypothesis gives
\begin{align*}
 \tilde d_m^{k+1}(s';\rho)
 &\le\left(1+\frac1H\right)
    \int_{s,a}\tilde d_m^k(s,a;\rho)P_*^k(s'\mid s,a)
       \II(s'\in\tilde{\cT}_m^{k+1}(\rho))\le\kappa_{k+1}d_m^{k+1}(s';\rho).
\end{align*}
Outside the trusted set both sides are zero. Multiplication by the same $\rho^{k+1}(a'\mid s')$ proves the state-action comparison and closes the induction. Finally average over a mixture and integrate the nonnegative test functions $|\inner{u}{\phi_j}|$ and $\inner{u}{\phi_j}^2$. This gives both conclusions, with the stated looser factor $(1+1/H)^{2(j-1)}$.
\end{proof}

\begin{proof}[Proof of Lemma \ref{lemma:perepoch_valuefunc_estimationerror_linear}]
We first work with a base policy $\rho$. Apply the local simulation identity to $\hat M_m^{out}$ and $M_*$, retaining the true-model continuation value $V_*^{k+1}(s';\rho)\in[0,1]$. Split each estimated occupancy into its trusted part and its discarded part. The absolute one-step transition contribution is at most one, as is the absolute reward difference. Therefore the discarded parts contribute at most
\begin{align*}
 2\sum_{k=1}^H\int_{s,a}
       \bigl(\hat d_m^k(s,a;\rho)-\tilde d_m^k(s,a;\rho)\bigr)\le 2H^2c_{\cM}\zeta_m^{-1/2}B_m(\rho)^{1/4}=2H^2c_{\cM}\sqrt{H+1}\,v_m^{1/4}B_m(\rho)^{1/4}.
\end{align*}

For the retained parts, fix $k$, and use the notation $U_k$ and $\Delta\mu_k$ from the preceding proof. Let
\[
 a_{s'}=\|\Delta\mu_k(s')\|_{U_k}.
\]
Lemma~\ref{lemma:tilde_dP<dP*_linear} and the oracle event give
\begin{align*}
 \int_{s'}a_{s'}^2
 &\le e^2E_m+\beta_m\int_{s'}\|\Delta\mu_k(s')\|_2^2\le e^2E_m+4c_{\cM}^4\beta_m=v_m.
\end{align*}
The last inequality follows from
\[
 \int_{s'}\|\Delta\mu_k(s')\|_2
 \le2c_{\cM}^2,
 \ 
 \int_{s'}\|\Delta\mu_k(s')\|_2^2
 \le\left(\int_{s'}\|\Delta\mu_k(s')\|_2\right)^2.
\]
Moreover, $\|U_k\|_{\mathrm{op}}\le1+\beta_m\le2$, so the coefficient normalization implies
\[
 \int_{s'}\sqrt{a_{s'}}
 \le2^{1/4}\int_{s'}\sqrt{\|\Delta\mu_k(s')\|_2}
 \le2c_{\cM}2^{1/4}.
\]
H\"older's inequality with exponents $3/2$ and $3$ now gives the required summability estimate on the countable state space:
\begin{align*}
 \int_{s'}a_{s'}
 &=\int_{s'}(\sqrt{a_{s'}})^{2/3}(a_{s'}^2)^{1/3}\le\left(\int_{s'}\sqrt{a_{s'}}\right)^{2/3}
       \left(\int_{s'}a_{s'}^2\right)^{1/3}
 \le(2c_{\cM})^{2/3}2^{1/6}v_m^{1/3}.
\end{align*}

Set $z_k(\rho)=\int_{s,a}\phi_k(s,a)\tilde d_m^k(s,a;\rho)$. As in the preceding proof,
\[
 \|z_k(\rho)\|_{U_k^{-1}}\le\sqrt{B_m(\rho)}.
\]
Since the continuation value remains bounded by one, the retained transition contribution at layer $k$ is at most
\begin{align*}
 &\left|\inner{z_k(\rho)}{
        \int_{s'}\Delta\mu_k(s')V_*^{k+1}(s';\rho)}\right|\le\sqrt{B_m(\rho)}\int_{s'}a_{s'}
 \le(2c_{\cM})^{2/3}2^{1/6}v_m^{1/3}\sqrt{B_m(\rho)}.
\end{align*}
Likewise, writing $\Delta\theta_k=\hat\theta_{m,k}^k-\theta_*^k$, the oracle event and $\|\Delta\theta_k\|_2\le2$ imply
\[
 \|\Delta\theta_k\|_{U_k}^2
 \le e^2E_m+4\beta_m\le v_m.
\]
Thus the retained reward contribution is at most $\sqrt{v_mB_m(\rho)}$. Summing these bounds over $k$ proves the displayed base-policy estimate.

For an episode mixture $p$, apply the simulation identity separately to each support policy $\rho$ and average. The triangle inequality, followed by Jensen's inequality, gives
\[
 \EE_{\rho\sim p}[B_m(\rho)^\alpha]
 \le B_m(p)^\alpha,\  \alpha\in\{1/4,1/2\}.
\]
This proves the first stated inequality for every mixture. If $v_m\le1$, then $v_m^{1/3},v_m^{1/2}\le v_m^{1/4}$; also $B_m(p)\ge Hd\ge1$, so both its quarter and half powers are at most $B_m(p)$. The sum of the displayed coefficients is at most $C_0=100(1+c_{\cM})^4H^3$. Hence
\[
 |\hat V_m^1(p)-V_*^1(p)|
 \le C_0v_m^{1/4}B_m(p).
\]
If $v_m>1$, the same last bound follows from the normalized total reward bound $|\hat V_m^1(p)-V_*^1(p)|\le1$. Combining the two bounds proves the final assertion.
\end{proof}

\begin{proof}[Proof of Lemma~\ref{lemma:pseudo_regret_linear}]
Optimality of $p_m^h$ in the segment-$h$ objective, with comparator $p_{m-1}\in\Pi_{\mathrm{mix}}$, gives
\[
\eta_m\hatreg_{m-1}(p_m^h)
\le\Lambda_m^h(p_m^h)-\Lambda_m^h(p_{m-1}).
\]
For each $j\le h$, the trusted occupancy has mass at most one and $\|\phi_j\|_2\le1$, so $\tr(\tilde K_m^j(p_m^h))\le1$. Applying the arithmetic--geometric mean inequality to the eigenvalues, we obtain
\[
\log\det(\tilde K_m^j(p_m^h)+\beta_mI_d)
\le d\log(1/d+\beta_m).
\]
The benchmark covariance is positive semidefinite, hence
$\log\det(\tilde K_m^j(p_{m-1})+\beta_mI_d)\ge d\log\beta_m$.
Summing these bounds over $j\le h$ and dividing by $\eta_m$ yields
\[
\hatreg_{m-1}(p_m^h)
\le\frac{hd}{\eta_m}\log(1/d+\beta_m)-\frac{hd}{\eta_m}\log\beta_m.
\]
\end{proof}

\section{Proofs in Subsection \ref{subsec:theory_regret_analysis_linear}}

\begin{proof}[Proof of Theorem \ref{thm:main_reg_linear}]
Condition on the simultaneous offline-oracle event. Write $D=Hd$, and let $p\in\Pi_{\mathrm{mix}}$ denote an arbitrary meta-policy. The corrected value-estimation bound, with the policy-specific trusted sets and $\zeta_m^{-1}=(H+1)\sqrt{v_m}$, gives, whenever $v_m\le1$,
\[
 |\hat V_m^1(p)-V_*^1(p)|
 \le C_0v_m^{1/4}\bigl(D+\eta_m\hatreg_{m-1}(p)\bigr).
\]
Here Lemma~\ref{lemma:perepoch_valuefunc_estimationerror_linear} is applied with $B_m(p)=D+\eta_m\hatreg_{m-1}(p)$; its constant is bounded by the stated $C_0$.

The sequence $a_m$ in the theorem is nonincreasing. Set $\alpha_m=a_m/(100a_{m-1})\le1/100$. If $a_m<1$, the preceding bound and the definition of $\eta_m$ imply
\[
 |\hat V_m^1(p)-V_*^1(p)|
 \le a_m+\alpha_m\hatreg_{m-1}(p).
\]
If $a_m=1$, the same inequality follows directly from the unit value bound. This also covers $v_m>1$, because then $a_m=1$.

Lemma~\ref{app-lemma:lagged_calibration} therefore gives, uniformly over $p$,
\[
\hatreg_m(p)\le2\mathrm{reg}(p)+10a_m,
\  \mathrm{reg}(p)\le2\hatreg_m(p)+10a_m.
\]

Apply the second calibration inequality at epoch $m-1$ to the executed mixture $p_m^h$. Lemma~\ref{lemma:pseudo_regret_linear} yields
\begin{align*}
 \mathrm{reg}(p_m^h)
 &\le2\hatreg_{m-1}(p_m^h)+10a_{m-1}\\
 &\le\frac{2hd}{\eta_m}\log\!\left(1+\frac1{d\beta_m}\right)+10a_{m-1}
 \le210Hd\,L_m a_{m-1}.
\end{align*}
The gap is also at most one. Summing over the $H$ segments, each of length $n_m$, proves the stated regret bound. The joint oracle event has probability at least $1-\delta$, since the per-call confidence levels $\delta_m$ satisfy $2H\sum_{m=1}^N\delta_m\le\delta$ under either confidence allocation in the theorem. Each segment uses one call to each oracle, and each epoch uses one additional unregularized planning call to compute the estimated optimal value entering the trusted-set thresholds. Thus the exact counts are at most $NH$ estimation calls and $(H+1)N$ planning calls.
\end{proof}

\begin{proof}[Proof of Corollary \ref{cor:reg_known_T_linear}]
Write $B=T/H$ and $\ell_N=\max\{1,80H\log(4H|\cM|(N+1)^2/\delta)\}$. The finite-class least-squares bound gives
\[
 a_m\lesssim C_0Hd(1+c_\cM)\ell_N^{1/4}n_m^{-1/4},
 \; L_N\lesssim\log(2+T).
\]
Let $M$ be the last index for which $\tau_M=2B^{1-2^{-M}}\le B/2$, and append $\tau_{M+1}=B$. Then $N=M+1=O(\log\log B)$. Before the appended endpoint, $n_m\ge\tau_m/2$. For $2\le m\le M$,
\[
 n_m n_{m-1}^{-1/4}
 \le2^{1/4}\tau_m\tau_{m-1}^{-1/4}
 \le C B^{3/4-2^{-m-1}}
 \le C B^{3/4}.
\]
The first epoch has length $2\sqrt B\le2B^{3/4}$. Maximality of $M$ gives $\tau_{M+1}^{\mathrm{untruncated}}>B/2$; since $\tau_M=(\tau_{M+1}^{\mathrm{untruncated}})^2/(2B)$, this implies $\tau_M>B/8$ and hence $n_M>B/16$. Therefore the appended epoch contributes at most $CB^{3/4}$ as well. Theorem~\ref{thm:main_reg_linear} now gives
\begin{align*}
 \mathrm{Reg}(T)
 &\lesssim C_0(1+c_\cM)H(Hd)^2\ell_N^{1/4}L_N\,N B^{3/4}\\
 &\lesssim C_0(1+c_\cM)H^{1/4}(Hd)^2\ell_N^{1/4}\log(2+T)\,N T^{3/4}.
\end{align*}
Integer rounding affects only absolute constants. The bounded-$B$ cases follow from $\mathrm{Reg}(T)\le T$. The oracle-call count follows from the number of segments.
\end{proof}

\begin{proof}[Proof of Corollary \ref{cor:reg_unknown_T_linear}]
For $\tau_m=2^m$, $n_m=2^{m-1}$ for $m\ge2$, with an immaterial change at $m=1$. Consequently,
\[
 \sum_{m=2}^N n_m n_{m-1}^{-1/4}
 \lesssim\sum_{m=2}^N2^{3m/4}
 \lesssim(T/H)^{3/4}.
\]
The initial epoch has bounded length. The finite-class bounds on $a_m$ and $L_m$ used in the previous proof and Theorem~\ref{thm:main_reg_linear} therefore give
\[
 \mathrm{Reg}(T)
 \lesssim C_0(1+c_\cM)H^{1/4}(Hd)^2\ell_N^{1/4}\log(2+T)\,T^{3/4}.
\]
If the horizon stops inside an epoch, charge only the executed trajectories in that epoch; its planned endpoint is within a constant factor of the current horizon. This does not change the bound. Since $N=O(\log T)$, there are $O(H\log T)$ calls to each oracle.
\end{proof}

\end{document}